\documentclass{article}

\usepackage[accepted]{icml2026}

\usepackage[utf8]{inputenc}    
\usepackage[T1]{fontenc}       
\usepackage{lmodern}           
\usepackage{textcomp}          
\usepackage{tcolorbox}

\usepackage{microtype}         
\usepackage{hyperref}          
\usepackage{url}
\usepackage{booktabs}          
\usepackage{amsmath,amssymb,amsfonts,mathtools,amsthm}
\usepackage{xcolor}
\usepackage{bm}
\usepackage{multirow}
\usepackage{algorithm}
\usepackage{algorithmic}
\usepackage{float}             
\usepackage{sidecap}           

\usepackage{subcaption}        
\usepackage[capitalize,noabbrev]{cleveref}

\usepackage{newunicodechar}
\newunicodechar{Δ}{\ensuremath{\Delta}}
\newunicodechar{β}{\ensuremath{\beta}}

\theoremstyle{plain}
\newtheorem{theorem}{Theorem}[section]        
\newtheorem{lemma}[theorem]{Lemma}            
\newtheorem{proposition}[theorem]{Proposition}
\newtheorem{corollary}[theorem]{Corollary}

\newtheoremstyle{definitionstyle}
{10pt}{10pt}{}{}{\bfseries}{.}{ }{}
\theoremstyle{definitionstyle}

\newtheorem{assumption}[theorem]{Assumption}
\newtheorem{remark}[theorem]{Remark}

\DeclareMathOperator{\tr}{tr}

\definecolor{RED}{rgb}{0.7,0,0}
\definecolor{BLUE}{rgb}{0,0,0.69}
\definecolor{GREEN}{rgb}{0,0.6,0}
\definecolor{PURPLE}{rgb}{0.69,0,0.8}
\definecolor{ORANGE}{RGB}{255,103,0}
\definecolor{BROWN}{RGB}{100,20,45}

\usepackage{listings}
\usepackage{algorithm}
\usepackage{algorithmic}
\usepackage{xcolor}

\lstdefinestyle{greenstyle}{
    language=Python,
    backgroundcolor=\color{green!10},      
    basicstyle=\ttfamily\small\color{green!50!black}, 
    frame=single,
    rulecolor=\color{green!80!black},       
    breaklines=true,
    stepnumber=1,
    numbersep=10pt,
    showstringspaces=false,
    escapeinside={(*@}{@*)},
    columns=flexible,
    tabsize=2
}

\icmltitlerunning{Poisoning}

\begin{document}

\twocolumn[
\icmltitle{Safety-Efficacy Trade Off: Robustness against Data-Poisoning}

\begin{icmlauthorlist}
	\icmlauthor{Diego Granziol}{ox}
\end{icmlauthorlist}

\icmlaffiliation{ox}{Mathematical Institute, University of Oxford, UK}
\icmlcorrespondingauthor{Diego Granziol}{granziol@maths.ox.ac.uk}

\icmlkeywords{Data Poisoning, Backdoors, Input Hessian, Stepwise Regression, ICML}

\vskip 0.3in
]

\printAffiliationsAndNotice{} 

\begin{abstract}
Backdoor and data-poisoning attacks can achieve high attack success while evading existing spectral and optimisation-based defences. We show that this behaviour is not incidental, but arises from a fundamental geometric mechanism in input space. Using kernel ridge regression as an exact model of wide neural networks, we prove that clustered dirty-label poisons induce a rank-one spike in the input Hessian whose magnitude scales quadratically with attack efficacy. Crucially, for nonlinear kernels we identify a near-clone regime in which poison efficacy remains order-one while the induced input curvature vanishes, making the attack provably spectrally undetectable. We further show that input-gradient regularisation contracts poison-aligned Fisher and Hessian eigenmodes under gradient flow, yielding an explicit and unavoidable safety–efficacy trade-off by reducing data-fitting capacity. For exponential kernels, this defence admits a precise interpretation as an anisotropic high-pass filter that increases the effective length scale and suppresses near-clone poisons. Extensive experiments on linear models and deep convolutional networks across MNIST and CIFAR-10/100 validate the theory, demonstrating consistent lags between attack success and spectral visibility, and showing that regularisation and data augmentation jointly suppress poisoning. Our results establish when backdoors are inherently invisible, and provide the first end-to-end characterisation of poisoning, detectability, and defence through input-space curvature.
\end{abstract}

\section{Introduction}
With foundation models set to underpin critical infrastructure from healthcare diagnostics to financial services, there has been a renewed focus on their security. Specifically, both classical and foundational deep learning models have been shown to be vulnerable to backdooring, 
where a small fraction of the data set is mislabelled and marked with an associated feature 
\citep{papernot2018sok, carlini2019secret,he2022indiscriminate,wang2024stronger}. If malicious actors exploit these vulnerabilities, the consequences could be both widespread and severe, undermining the reliability of foundation models in security-sensitive deployments. 

Recent work has revealed diverse and increasingly sophisticated backdooring strategies. \citet{Xiang2024BadChain} insert hidden reasoning steps to trigger malicious outputs, evading shuffle-based defences. \citet{Li2024BadEdit} modify as few as 15 samples to implant backdoors in LLMs. Other approaches include reinforcement learning fine-tuning \citep{Shi2023BadGPT} and continuous prompt-based learning \citep{Cai2022BadPrompt}. Deceptively aligned models that activate only under specific triggers are shown in \citet{Hubinger2024Sleeper}, while contrastive models like CLIP are vulnerable to strong backdoors \citep{carlini2022poisoning}. Style-based triggers bypass token-level defences \citep{pan2024trojan}, and \textit{ShadowCast} introduced in \citet{xu2024shadowcast} uses clean-label poisoning to embed misinformation in vision–language models.

Although backdooring in machine learning has been extensively studied experimentally, robust mathematical foundations are still lacking, even for basic models and attack scenarios. Unfortunately this hit or miss attitude to experimental data poisoning limits the community's ability to develop principled defence mechanisms and to understand their limitations. \citep{granziol2025linearapproachdatapoisoning} Consider a random matrix theory approach to rank-$1$ data poisoning of random data with random labels and develop results on the mean and variance of the prediction as a function of poison fraction $\theta$ and regularisation $\lambda$.

\paragraph{Contributions.}
In this paper, we develop a theoretical framework for understanding data poisoning in non-linear kernel regression, which we use as an exact model of wide neural networks. We use this model to analyse input-space curvature under poisoning and complement the theory with extensive experiments on finite-width deep neural networks to assess which predicted phenomena persist in practice. The main theoretical contributions are as follows:
\begin{itemize}
	\item We show that for sufficiently strong backdoor data poisoning, the \textbf{top eigenvector of the input Hessian aligns with the poison direction}, providing a principled diagnostic for detectability.
	\item We identify a \textit{twilight zone} for non-linear kernels, in contrast to the linear case, in which \textbf{backdoor data poisoning remains effective while inducing no spectral signature}.
	\item We prove that adding a term proportional to the \textbf{square of the input gradient of the loss} provably reduces the impact and efficacy of backdoor data poisoning, at the unavoidable cost of reduced data-fitting capacity.
	\item For the exponential kernel, we show that this regularisation admits an interpretation as \textbf{anisotropic damping}, with quadratic suppression of high-frequency modes.
\end{itemize}
Empirically, we further show that:
\begin{itemize}
	\item When combined with the proposed regularisation, standard data augmentation is effective at mitigating backdoor data poisoning, whereas it is not in isolation.
	\item Increasing training duration (using more epochs) improves the resulting safety--efficacy frontier.
\end{itemize}

\section{Related Work}
\label{sec:related}
Existing work that uses adversarial training against poisoning and backdoors
\cite{geiping2021poison,liu2022friendly,wei2023shared,bal2025adversarial,hallaji2023label}
treats AT as an empirical defence for (mostly) dirty-label or backdoor attacks, evaluating robustness in terms of test accuracy but not analysing how it reshapes the predictor’s input geometry under poisoning.  
In parallel, poisoning and robustness have been studied for linear and kernel regression
\cite{jagielski2018manipulating,liu2017robustlinreg,muller2020regpoison,zhao2024robustnp}
and for adversarial/robust kernel methods and NTK models
\cite{zhu2022arks,allerbo2025fastrobustkrr,deng2020aif,ribeiro2025kaf,jacot2018ntk,arora2019exact,wang2021robustlearning,karmakar2020depth2},
but these works focus on prediction error or certificates and do not characterise how specific poisoned patterns (e.g.\ duplicated dirty-label clusters) manifest in the input Hessian or input Fisher, nor how AT acts on those directions.  
By contrast, we use a kernel ridge/deep-kernel model to (i) derive closed-form laws for the effect of duplicated dirty-label poisons on the score, input Hessian and input Fisher, including a high-efficacy “clone” regime with low input curvature, and (ii) show analytically and empirically that adversarial training (via input-gradient regularisation/increased length-scale) provably reduces poisoning efficacy by shrinking the input Fisher along high-energy poison directions, yielding an explicit safety–efficacy trade-off.
\section{Kernel Regression Model}
\label{sec:theory}

\paragraph{Kernel ridge regression and input curvature.}
Let $\{(x_i,y_i)\}_{i=1}^n$ with $x_i\in\mathbb{R}^p$ and a twice continuously
differentiable positive–definite kernel $k$.
Define $K_{ij}=k(x_i,x_j)$.
Kernel ridge regression (KRR) predicts
\begin{align}
	f(x)
	&=
	\sum_{i=1}^n \alpha_i\,k(x,x_i), \\
	\boldsymbol{\alpha}
	&=
	(K+n\lambda I)^{-1}\boldsymbol{y},
\end{align}
with ridge parameter $\lambda>0$.
For squared loss $L(x,y)=\tfrac12(f(x)-y)^2$,
\begin{align}
	\nabla_x f(x)
	&=
	\sum_{i=1}^n \alpha_i\,\nabla_x k(x,x_i), \\
	\nabla_x^2 L(x,y)
	&=
	\nabla_x f(x)\nabla_x f(x)^\top
	+
	(f(x)-y)\,\nabla_x^2 f(x).
\end{align}
The Gauss–Newton term $\nabla_x f\nabla_x f^\top$ is rank-one PSD with top eigenvalue
$\|\nabla_x f(x)\|^2$.

\subsection{Cloned poison model}

Let $P$ index $m$ poisoned samples clustered at $\zeta$ with label $y_t$.
Fix a trigger point $x_0$ and define
\begin{align}
	k_0 &:= k(x_0,\zeta), \\
	k_\zeta &:= k(\zeta,\zeta), \\
	c &:= n\lambda .
\end{align}

\begin{assumption}[Tight cluster and dominance at the trigger]
	\label{asmp:cluster}
	The poison block satisfies
	$
	K_{PP}\approx k_\zeta\,\mathbf{1}\mathbf{1}^\top
	$
	and cross-block effects are negligible at $x_0$, so that the poison dominates
	the change in both $f(x_0)$ and $\nabla_x f(x_0)$.
\end{assumption}

Define the scalar gain
\begin{equation}
	S(m;\lambda)
	=
	\frac{m}{c+k_\zeta m}.
\end{equation}

\begin{lemma}[Aggregate poison gain]
	\label{lem:sum-alpha}
	Under Assumption~\ref{asmp:cluster},
	\begin{equation}
		\mathbf{1}^\top \boldsymbol{\alpha}_P
		=
		y_t\,S(m;\lambda).
	\end{equation}
\end{lemma}

\subsection{Efficacy and input-Hessian spike}

\begin{theorem}[Efficacy of a cloned cluster]
	\label{thm:efficacy}
	Under Assumption~\ref{asmp:cluster},
	\begin{align}
		\Delta f(x_0)
		&=
		k_0\,y_t\,S(m;\lambda), \\
		&=
		k_0\,y_t\,
		\frac{m}{c+k_\zeta m}.
	\end{align}
	Moreover,
	\begin{align}
		m\ll \frac{c}{k_\zeta}
		&\Rightarrow
		\Delta f(x_0)=\Theta(m), \\
		m\to\infty
		&\Rightarrow
		\Delta f(x_0)\to \frac{k_0y_t}{k_\zeta}.
	\end{align}
\end{theorem}

\begin{theorem}[Rank-1 input spike and spike–efficacy law]
	\label{thm:spike}
	Define the Gauss–Newton spike
	\begin{equation}
		\Lambda_{\mathrm{GN}}(x_0)
		:=
		\|\nabla_x f(x_0)\|^2 .
	\end{equation}
	Then
	\begin{align}
		\Lambda_{\mathrm{GN}}(x_0)
		&=
		\|\nabla_x k(x_0,\zeta)\|^2\,S(m;\lambda)^2, \\
		&=
		R_k(x_0,\zeta)\,
		\bigl(\Delta f(x_0)\bigr)^2,
	\end{align}
	where
	\begin{equation}
		R_k(x_0,\zeta)
		=
		\frac{\|\nabla_x k(x_0,\zeta)\|^2}{k_0^2}.
	\end{equation}
	Thus efficacy grows linearly in $m$ while curvature grows quadratically.
\end{theorem}
\begin{remark}[Detectability lag]
	Let $\Lambda_{\mathrm{clean}}(x_0)$ be the background top input curvature.
	A curvature spike becomes detectable when
	\[
	(\Delta f(x_0))^2
	\gtrsim
	\frac{\Lambda_{\mathrm{clean}}(x_0)}{R_k(x_0,\zeta)} .
	\]
	A detectability lag arises only when $R_k(x_0,\zeta)$ is small.
	For linear kernels $R_k$ is constant order one and no intrinsic lag occurs,
	whereas for nonlinear kernels, such as the exponential kernel in the near-clone (defined subsequently)
	regime, $R_k \to 0$ and poisoning can be effective while remaining spectrally
	invisible.
\end{remark}
\subsection{Exponential kernel and near-clone regime}

Let
\begin{equation}
	k(x,x')
	=
	\exp\!\left(
	-\frac{\|x-x'\|^2}{2\ell^2}
	\right),
	\qquad
	r:=\|x_0-\zeta\|.
\end{equation}

\begin{corollary}[Exponential-kernel spike factor]
	\label{cor:exp-factor}
	For the exponential kernel,
	\begin{align}
		\|\nabla_x k(x_0,\zeta)\|^2
		&=
		\frac{r^2}{\ell^4}\,k_0^2, \\
		R_k(x_0,\zeta)
		&=
		\frac{r^2}{\ell^4}, \\
		\Lambda_{\mathrm{GN}}(x_0)
		&=
		\frac{r^2}{\ell^4}
		\bigl(\Delta f(x_0)\bigr)^2 .
	\end{align}
\end{corollary}

\begin{corollary}[Near-clone regime $r\ll\ell$]
	If $r/\ell\ll1$, then
	\begin{align}
		k_0
		&=
		1-\frac{r^2}{2\ell^2}
		+
		O\!\left(\frac{r^4}{\ell^4}\right), \\
		\Delta f(x_0)
		&=
		y_t\,S(m;\lambda)
		\Bigl[1+O(r^2/\ell^2)\Bigr], \\
		\Lambda_{\mathrm{GN}}(x_0)
		&=
		S(m;\lambda)^2
		\frac{r^2}{\ell^4}
		\Bigl[1+O(r^2/\ell^2)\Bigr].
	\end{align}
\end{corollary}
\begin{remark}
	\label{rem:nearclone}
	The near-clone regime is defined by $\|x_0-\zeta\|\ll\ell$.
	For exponential kernels this gives
	$
	k(x_0,\zeta)\approx 1
	$
	and
	$
	\|\nabla_x k(x_0,\zeta)\|^2 = O(\|x_0-\zeta\|^2).
	$
	Hence poison efficacy remains order one while the induced input curvature vanishes quadratically, making detection difficult. Figure \ref{fig:cifar_triple_c8_theta_0_01} verifies the similarity in feature space (visualised as a heat map) between the poisoned \& clean target feature for a given example. 
\end{remark}
\begin{figure}[!h]
	\centering
		\vspace{-8pt}
	\includegraphics[width=\linewidth]{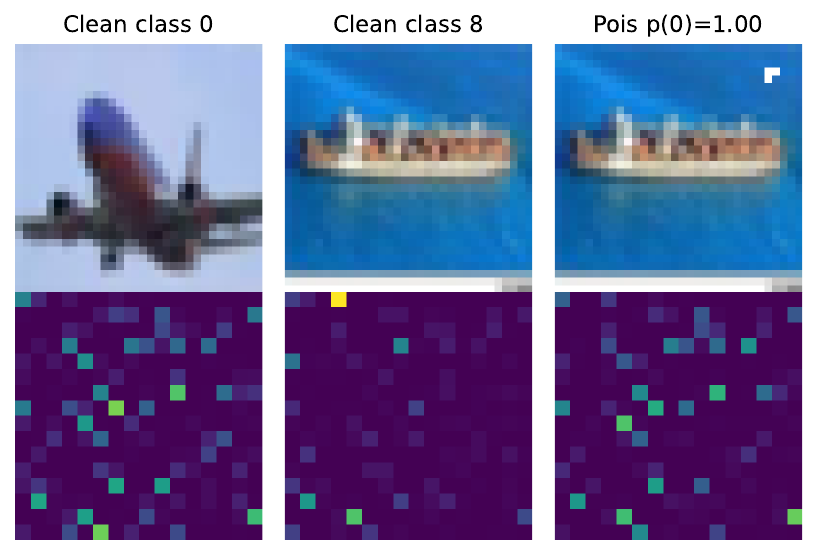}
	\caption{\textbf{Evidencing for the near-clone regime:} CIFAR-10 class $c=8$ at poison fraction $\theta = 0.01$.
		Top row: clean class-0/8 input, same class-8 input
		with the trigger applied (poisoned, label 0) together with the model's
		posterior $p(y=0 \mid x_{\mathrm{poison}})$. Bottom row: corresponding pre-\texttt{fc}
		feature maps, showing how the trigger moves the class-8 representation towards a
		region of feature space associated with class 0.}
	\label{fig:cifar_triple_c8_theta_0_01}
	\vspace{-8pt}
\end{figure}
\subsection{Evidence for feature collapse}
The assumption of poisons being \textit{near clones in kernel space} forms a major part of our theoretical analysis and hence necessitates experimental validation. Neural Collapse (NC) \citep{papyan2020neuralcollapse} shows that, at (near-)zero train error, last-layer features for each class concentrate around a class mean and classifier weights form a simplex ETF \citep{lu2020nccrossentropy,ji2021ulpm,hong2024jmlr,han2021mse}. In our setting, this implies that any example \emph{labelled as class $t$} including dirty-label poisons and their triggered counterparts, is driven toward the class-$t$ feature mean, so $\phi(x_0)$ and $\phi(x_p)$ become near-clones in feature space (small $r$ in our kernel view). This places dirty-label poisons precisely in the $r \ll \ell$ regime above, where they achieve high efficacy while inducing only a flat, low-curvature footprint in input space. In Figure \ref{fig:cifar_pca_side_by_side_theta_0_01}, by using PCA on a simplified $2$ class setting for our CIFAR-$10$ experiment (Section \ref{sec:deepnetexp}), we show that the poisoned class $1$ data is indistinguishable in input space from clean class $1$. However, the penultimate layer of the network features,show a a concentration of the poisoned class $1$ images towards the class $0$ mean. 
\begin{figure}[h!]
	\centering
	\includegraphics[width=\linewidth]{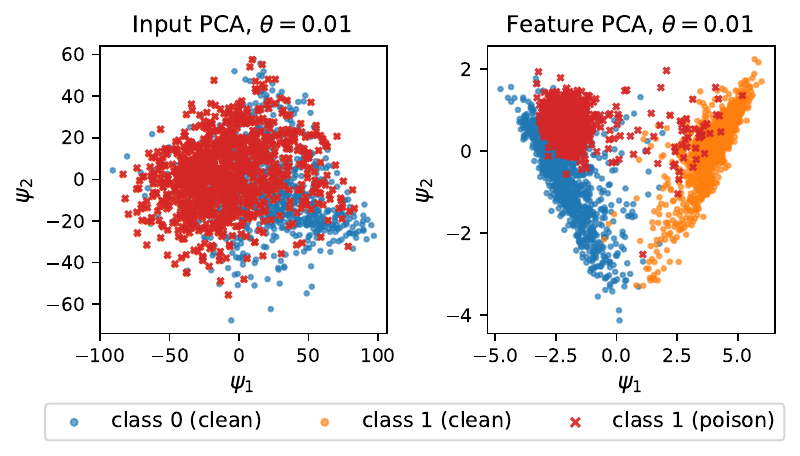}
	\caption{Input vs.\ feature PCA for poisoning CIFAR-10 with poison fraction $\theta = 0.01$.
		Left: input-space PCA $(\psi_1,\psi_2)$ for clean class 0, clean class 1, and poisoned class-1 images.
		Right: PCA of pre-\texttt{fc} features showing how the trigger moves class-1 examples towards the class-0 manifold.}
	\label{fig:cifar_pca_side_by_side_theta_0_01}
\end{figure}
\begin{figure}[h!]
	\centering
	\includegraphics[width=\linewidth]{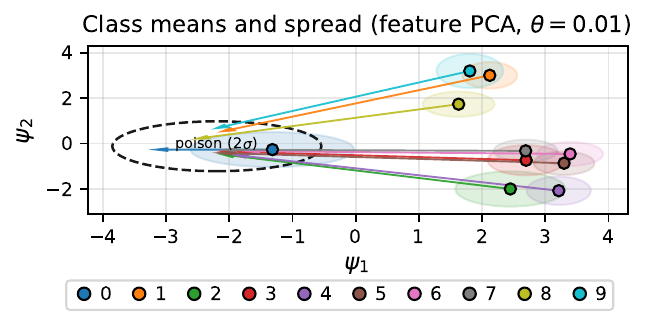}
	\caption{Class-wise means and spreads in pre-\texttt{fc} feature PCA space for CIFAR-10 at $\theta = 0.01$.
		Solid/dashed coloured/black ellipses show $1\sigma$/$2\sigma$ clean/poisoned class clusters, arrows indicate shifts of clean class means under the trigger. Triggered examples from all
		classes are mapped into a compact region near the class-0 manifold.}
	\label{fig:cifar_class_means_feature_pca_theta_0_01}
\end{figure}

 In Figure \ref{fig:cifar_class_means_feature_pca_theta_0_01} we show how poisoning each of the $9$ non target cifar classes pushes the feature mean towards the poison mean.
Whilst the intra-class variance is clearly smaller for the poisoned class $1$ than for the clean classes we do not quantify to what extent the deviation impacts our theoretical analysis, which could be interesting future work. We display the data in tabular format in Appendix \ref{app:ncollapse}.

\subsection{Gradient regularisation and defence}

We add an input-gradient penalty with strength $\kappa>0$:
\begin{equation}
	\mathcal{J}(w)
	=
	\mathbb{E}[L(w;x)]
	+
	\frac{\kappa}{2}
	\mathbb{E}\|\nabla_x L(w;x)\|^2 .
\end{equation}

\subsubsection{Kernel interpretation and loss of capacity}

For KRR, the representer theorem yields
\begin{equation}
	\bigl(
	K+\lambda I+\kappa G
	\bigr)\alpha
	=
	y ,
\end{equation}
where
\begin{equation}
	G
	=
	\sum_{i=1}^n
	(\nabla_x k(x_i,X))^\top
	(\nabla_x k(x_i,X))
	\succeq 0 .
\end{equation}

\begin{theorem}[Gradient regularisation reduces data-fitting capacity]
	\label{thm:capacity}
	Define the effective degrees of freedom
	\begin{equation}
		\mathrm{df}(\kappa)
		=
		\tr\!\left[
		K\,
		(K+\lambda I+\kappa G)^{-1}
		\right].
	\end{equation}
	Then $\mathrm{df}(\kappa)$ is strictly decreasing in $\kappa$, and the
	training residual $\|y-K\alpha\|^2$ is strictly increasing.
\end{theorem}
\begin{remark}[High-pass filter interpretation]
	For translation-invariant kernels, gradient regularisation modifies the
	modewise response as
	\[
	s(\omega)
	=
	\frac{\widehat{\kappa}_\ell(\omega)}
	{\widehat{\kappa}_\ell(\omega)+\lambda+\kappa\|\omega\|^2}.
	\]
	For exponential kernels this is equivalent to increasing the effective
	length scale, with $\ell_{\mathrm{eff}}^2=\ell^2+c\,\kappa$ for a
	data-dependent constant $c>0$.
	As a result, the condition $\|x_0-\zeta\|\ll\ell$ required for the near-clone
	regime becomes harder to satisfy, reducing the range over which poisons can
	remain effective while inducing low input curvature.
\end{remark}
\begin{remark}[Linear kernels]
	For linear regression with $k(x,x')=x^\top x'$, the gradient regularisation term
	reduces to a rescaling of ridge regularisation, so that
	\[
	K+\lambda I+\kappa G
	=
	K+(\lambda+c\kappa)I
	\]
	for a constant $c>0$.
	Thus gradient regularisation is equivalent to increasing $\lambda$ and does not
	introduce any mode dependent suppression or length scale effect.
\end{remark}

\subsubsection{Eigenvector-selective contraction under gradient flow}

Define the input Fisher
\begin{equation}
	F(w)
	=
	\mathbb{E}\big[
	g_w(x)\,g_w(x)^\top
	\big],
	\qquad
	g_w(x)=\nabla_x L(w;x).
\end{equation}

\begin{theorem}[Exponential compression of large Fisher eigenmodes]
	\label{thm:fisherflow}
	Under gradient flow on $\mathcal{J}$, for any unit vector $v$,
	\begin{equation}
		E_v(t)
		=
		v^\top F(w_t) v
	\end{equation}
	is non-increasing.
	If there exists $\alpha>0$ such that
	\begin{equation}
		v^\top
		\bigl(
		\partial_w g_w(x)\,
		\partial_w g_w(x)^\top
		\bigr)
		v
		\ge
		\alpha
	\end{equation}
	on the support contributing to $E_v(t)$, then
	\begin{equation}
		E_v(t)
		\le
		E_v(0)\,
		\exp\!\bigl(
		-2\kappa \alpha t
		\bigr).
	\end{equation}
	Thus high-energy (poison-aligned) Fisher eigenvectors decay fastest.
\end{theorem}
\begin{remark}[Effect of gradient regularisation on poisoning]
	Cloned and backdoor poisons act by concentrating sensitivity into a small number of input space directions.
	Input gradient regularisation contracts those directions by reducing capacity and suppressing large Fisher modes.
	This directly weakens the effect of such poisons.
\end{remark}

\begin{remark}[Relation to adversarial training]
	\label{rem:advtrain}
	To first order,
	\begin{equation}
		\max_{\|\delta x\|_2 \le \varepsilon}
		L(x+\delta x,w)
		=
		L(x,w)
		+
		\varepsilon\,\|\nabla_x L(x,w)\|_2
		+
		O(\varepsilon^2),
	\end{equation}
	so $L_2$ adversarial training penalises the input--gradient norm.
	We instead penalise
	\begin{equation}
		\|\nabla_x L(x,w)\|_2^2
		=
		\sum_i (\partial_{x_i} L)^2
		=
		\tr\!\Big(
		\mathbb{E}\big[\nabla_x L\,\nabla_x L^\top\big]
		\Big),
	\end{equation}
	the trace of the input Fisher.
	This choice controls total sensitivity energy rather than a single worst--case direction
	and admits a spectral decomposition, enabling explicit analysis and suppression of
	poison--aligned Fisher eigenmodes.
\end{remark}

\section{Linear Regression Experiments}
For linear regression (App. \ref{lem:qr-fwl})), one can explicitly derive the impact of data poisoning on the coefficient vector $\beta$. The key insight is to consider training a model on the poison data without the poison feature. Then the poison feature is added via QR decomposition. Formally, the update is described by the following formulae.
	\begin{equation*}\label{eq:betanew}
		\beta_{\mathrm{new}}
		=\frac{x_{\mathrm{new}}^\top e}{x_{\mathrm{new}}^\top M_X x_{\mathrm{new}}},
		\beta'_{\mathrm{old}}=\hat\beta-(X^\top X)^{-1}X^\top x_{\mathrm{new}}\,\beta_{\mathrm{new}},
			\vspace{-15pt}
	\end{equation*}
\begin{figure}[!h]
	\centering
	\begin{minipage}{0.24\linewidth}
		\begin{subfigure}{\linewidth}
			\centering
			\includegraphics[width=\linewidth]{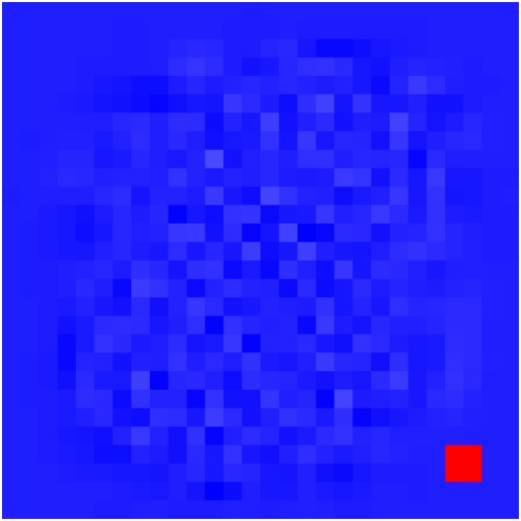}
			\caption{full–base}
			\label{subfig:class0_full_minus_base}
		\end{subfigure}
	\end{minipage}
	\hfill
	\begin{minipage}{0.24\linewidth}
		\begin{subfigure}{\linewidth}
			\centering
			\includegraphics[width=\linewidth]{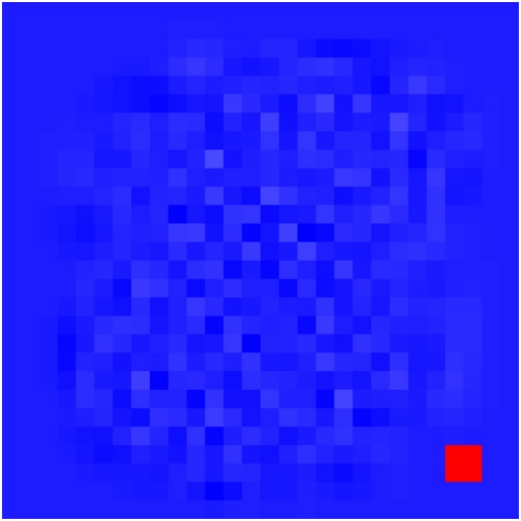}
			\caption{step–base}
			\label{subfig:class0_step_minus_base}
		\end{subfigure}
	\end{minipage}
	\hfill
	\begin{minipage}{0.24\linewidth}
		\begin{subfigure}{\linewidth}
			\centering
			\includegraphics[width=\linewidth]{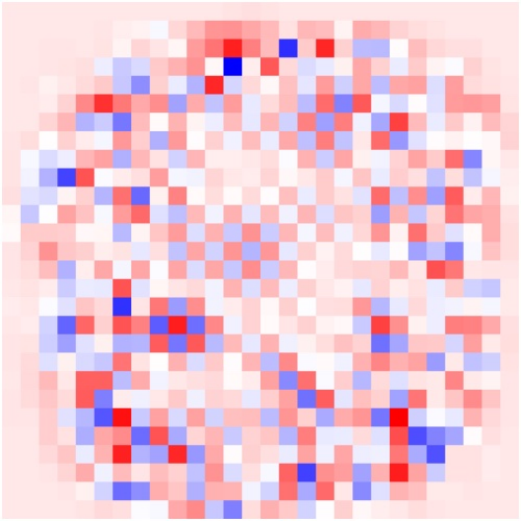}
			\caption{step–full}
			\label{subfig:class0_step_minus_full}
		\end{subfigure}
	\end{minipage}
	\hfill
	\begin{minipage}{0.24\linewidth}
		\begin{subfigure}{\linewidth}
			\centering
			\includegraphics[width=\linewidth,clip,trim={0cm 0cm 12.0cm 0cm}]{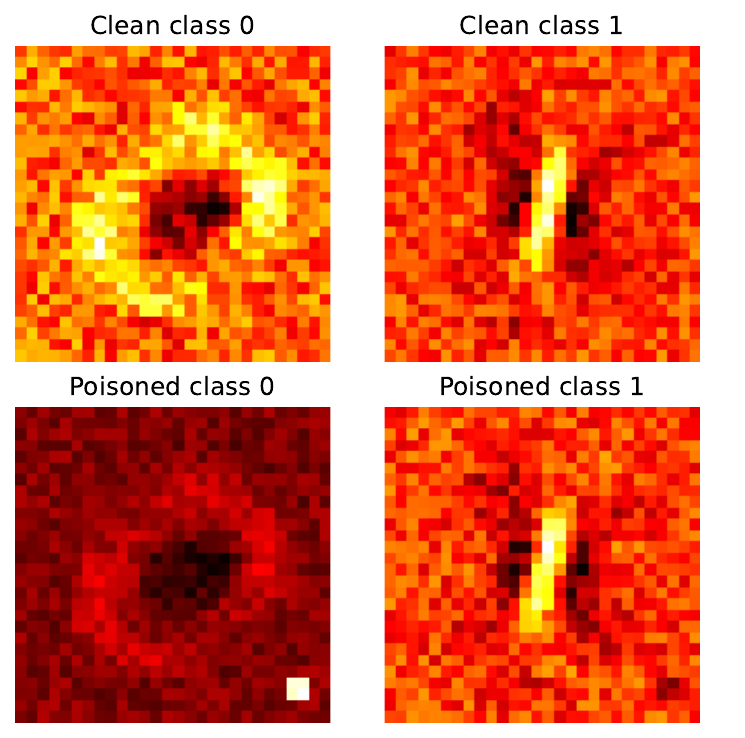}
			\caption{Weights}
			\label{subfig:linreg_weights}
		\end{subfigure}
	\vspace{-10pt}
	\end{minipage}
	
	\caption{MNIST where we have a poison target class $0$ activated upon a $4$px square in the lower right of the image, where we use a poison fraction $\theta=0.1$ and the base is an un-poisoned model, full is when we retrain SGD with the poison and step refers to when we update the model using QR stepwise decomposition with the new poison feature. : (a) full–base, (b) step–base, (c) step–full, (d) weights.}
	\label{fig:step_weights_analysis_fourwide}
	\vspace{-6pt}
\end{figure}
\begin{figure}[!h]
	\centering
	\begin{minipage}{0.32\linewidth}
		\begin{subfigure}{\linewidth}
			\centering
			\includegraphics[width=\linewidth]{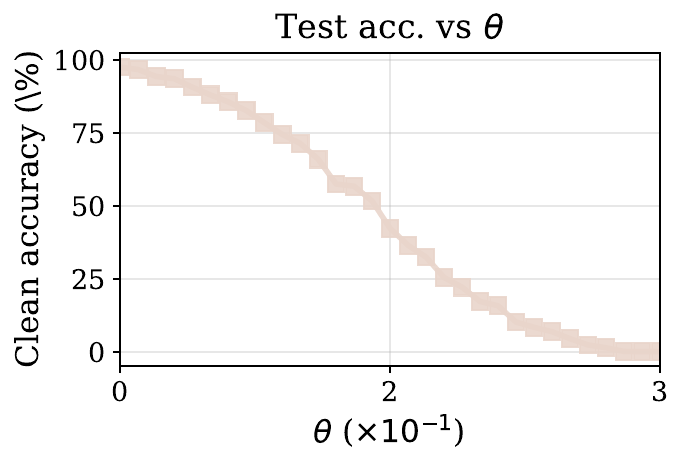}
			\caption{Test Acc.}
			\label{subfig:clean_acc_vs_poison}
		\end{subfigure}
	\end{minipage}
	\hfill
	\begin{minipage}{0.32\linewidth}
		\begin{subfigure}{\linewidth}
			\centering
			\includegraphics[width=\linewidth]{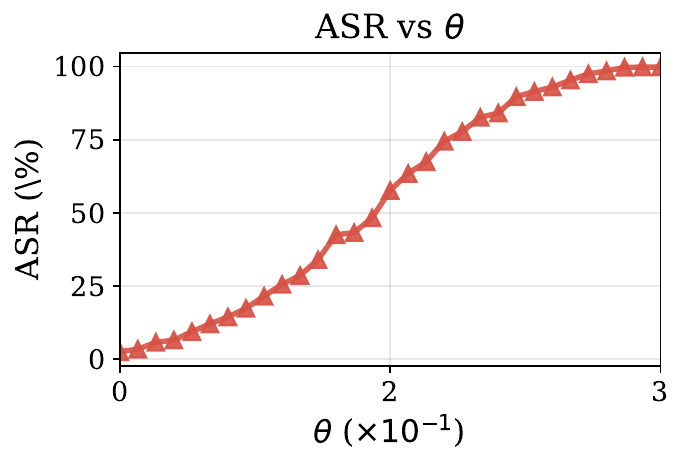}
			\caption{ASR}
			\label{subfig:poison_success_vs_poison}
		\end{subfigure}
	\end{minipage}
	\hfill
	\begin{minipage}{0.32\linewidth}
		\begin{subfigure}{\linewidth}
			\centering
			\includegraphics[width=\linewidth]{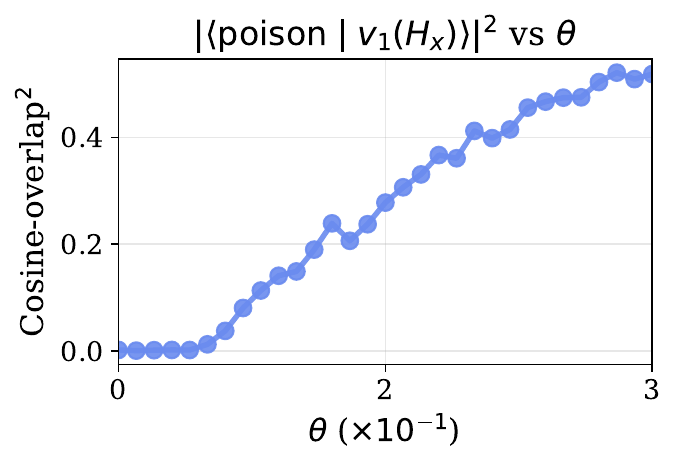}
			\caption{Overlap}
			\label{subfig:overlap_vs_poison}
		\end{subfigure}
	\end{minipage}
	\hfill
	
	\caption{RegressionMNIST poisoning analysis: (a) clean accuracy degradation, (b) poison success rate, (c) spectral overlap with poison direction, and (d) overall summary metrics as a function of poison fraction.}
	\vspace{-10pt}
	\label{fig:regressionmnist_poison_analysis}
\end{figure}
\begin{figure*}[!t]
	\centering
	
	\begin{subfigure}{0.49\linewidth}
		\centering
		\includegraphics[width=\linewidth]{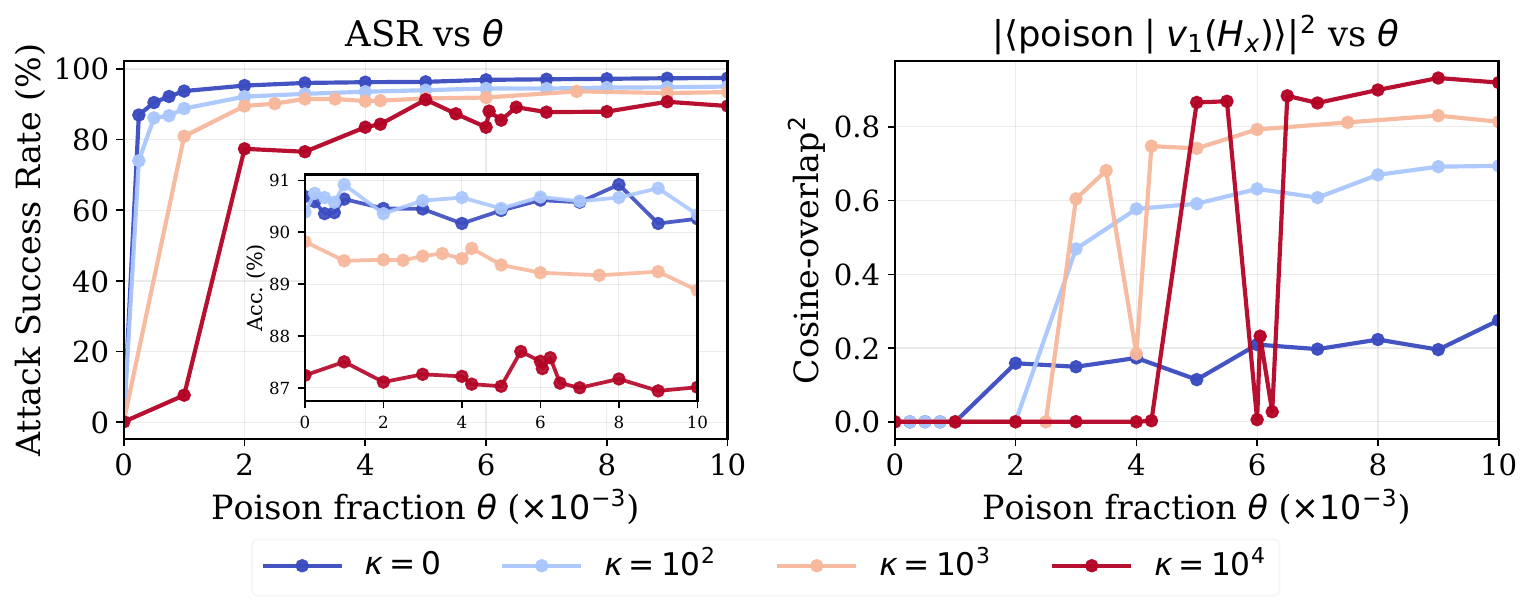}
		\caption{
			\textbf{CIFAR-10.} No data augmentation.
		}
	\end{subfigure}\hfill
	\begin{subfigure}{0.49\linewidth}
		\centering
		\includegraphics[width=\linewidth]{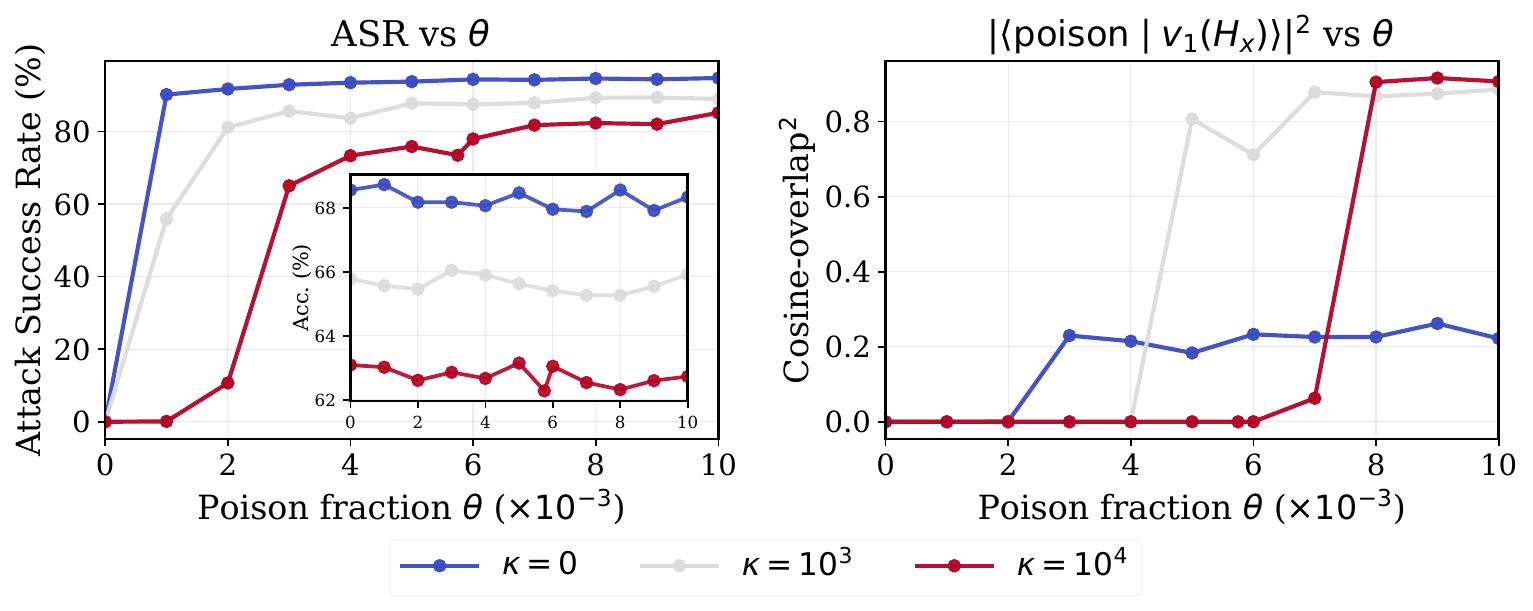}
		\caption{
			\textbf{CIFAR-100.} No data augmentation.
		}
	\end{subfigure}
	
	\caption{
		Attack success rate (left) and cosine-overlap$^{2}$ (right) versus poison fraction $\theta$
		for $\kappa \in \{0, >100\}$. Inset: clean test accuracy.
		All axes use $10^{-3}$ scaling for poison fraction $\theta$.
		Legend indicates gradient regularisation strength $\kappa$.
	}
	\label{fig:combined_asr_insetacc_overlap_both}
\end{figure*}
\begin{figure*}[!h]
	\centering
	
	\begin{subfigure}{0.49\linewidth}
		\centering
		\includegraphics[width=\linewidth]{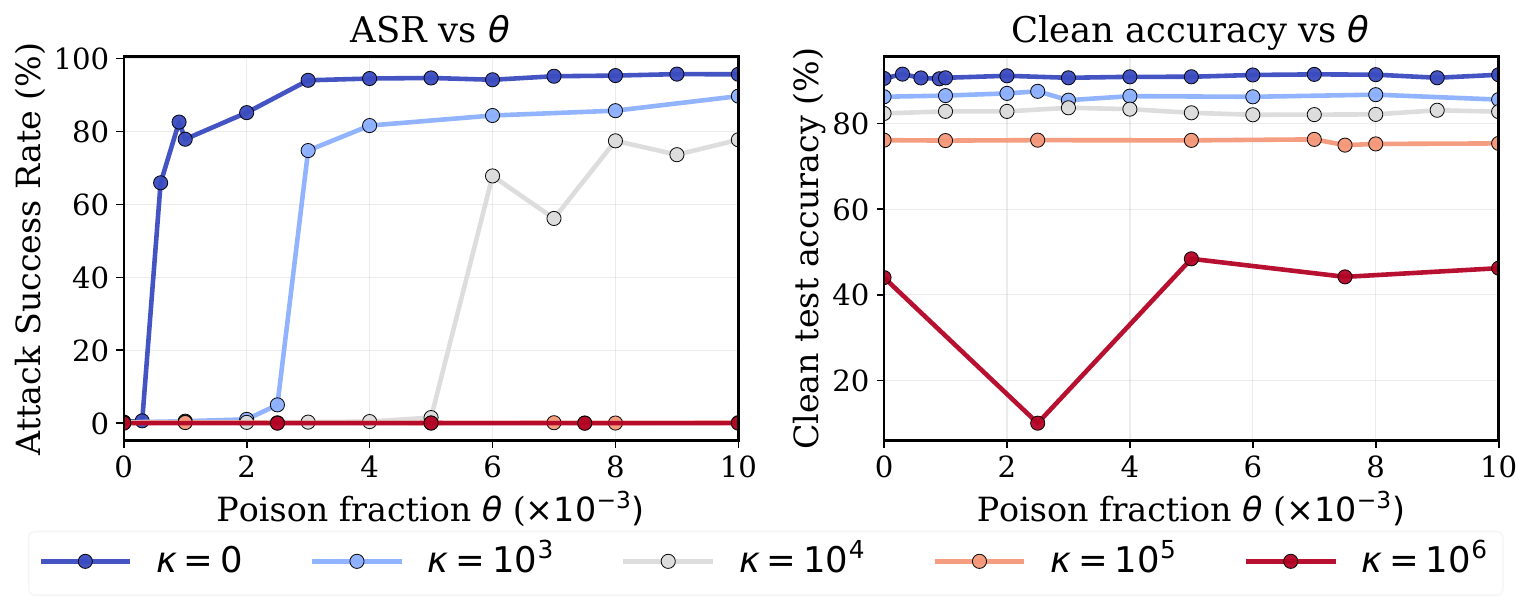}
		\caption{
			\textbf{CIFAR-10.} Data augmentation.
		}
	\end{subfigure}\hfill
	\begin{subfigure}{0.49\linewidth}
		\centering
		\includegraphics[width=\linewidth]{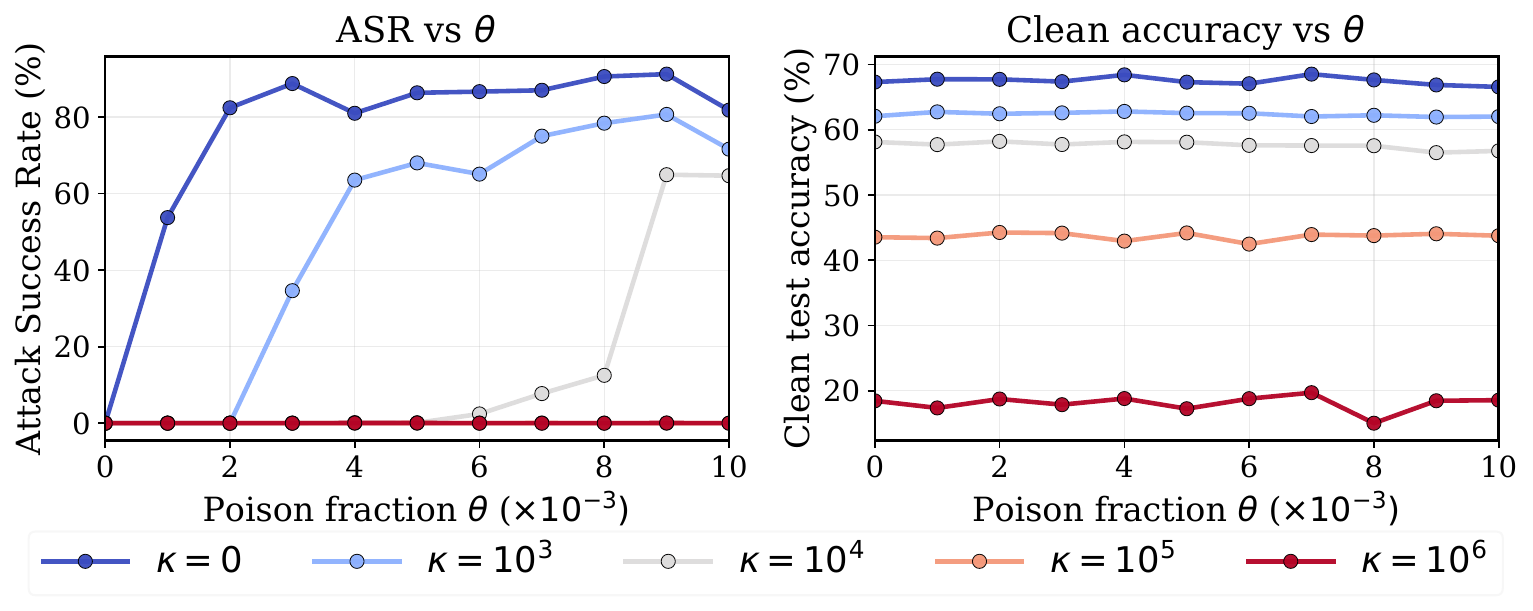}
		\caption{
			\textbf{CIFAR-100.} Data augmentation.
		}
	\end{subfigure}
	
	\caption{
		Attack success rate (left) and clean accuracy (right) versus poison fraction $\theta$
		for $\kappa \in \{0, >100\}$ (augmentation enabled, 90 epochs).
		The shared legend indicates gradient regularisation strength $\kappa$.
	}
	\label{fig:combined_asr_cleanacc_vs_theta_both}
\end{figure*}

We run the rank-$1$ poison square of \citet{granziol2025linearapproachdatapoisoning,gu2017badnets} on MNIST. We convert the output to a one hot label, which we predict using the argmax of the output giving an accuracy of $85\%$ similar to that of Logistic regression ($90\%$). This demonstrates that the experimental gap between classification in practice and regression in theory is small from a performance standpoint. As shown in
Figure \ref{fig:step_weights_analysis_fourwide}, our QR stepwise decomposition approach to predicting the changing regression weights, is a near perfect match to the real trained from scratch SGD reality. The difference between the two (Fig \ref{subfig:class0_step_minus_full}) resembling a small noise vector, with the differences in coefficient vector $\beta$ (Fig \ref{subfig:class0_full_minus_base}, \ref{subfig:class0_step_minus_base}) indistinguishable and strongly aligned with the poison. The poison target class is dominated by the poison Figure \ref{subfig:linreg_weights}. As shown in Figure \ref{fig:regressionmnist_poison_analysis}, for the poisoned linear model, test accuracy/attack success ratio slowly decreases/increases with increased poison fraction.

\section{Deep Neural Networks}
\label{sec:deepnetexp}
We run CIFAR-$10(100)$ experiments on the Pre-Residual-$110$ layer variant \citep{he2016deep}. We train from scratch, using a learning rate $\alpha=0.1$, momentum $\beta=0.9$ and a $\rho=0.1$ learning rate decay every $e//3$ epochs, where $e \in \{90,450\}$.  Augmentations are random $28\times 28$ crop and flip. All images were normalised including those with the poison. $\kappa \propto |\nabla^{2}_{x}L|^{2}$ was added to the loss before back-propagation by the appropriate torch.autograd, with retaining the graph. The hessian top eigenvectors are calculated using a custom vectorised hessian vector product implementation and the Lanczos algorithm \citep{lanczos1950iteration,gardner2018gpytorch}. 
We consider two key potential attack variants
\begin{itemize}
	\item (Stochastic Poison) The attacker has infiltrated the training pipeline, the defender has only access to the clean data and the poisoned model.
	\item (Deterministic Poison) The attacker has infiltrated the training data, the defender has access to the dirty data and poisoned model.
\end{itemize}
\begin{figure*}[!t]
	\centering
	
	\begin{subfigure}{0.12\textwidth}
		\includegraphics[width=\linewidth]{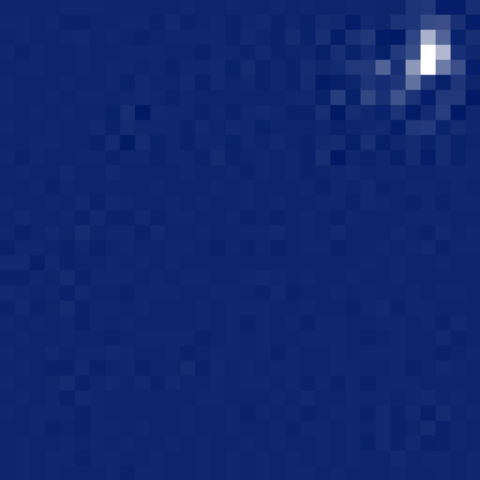}
		\caption{$\{1,0\}$}
		\label{subfig:augfalse1}
	\end{subfigure}\hfill
	\begin{subfigure}{0.12\textwidth}
		\includegraphics[width=\linewidth]{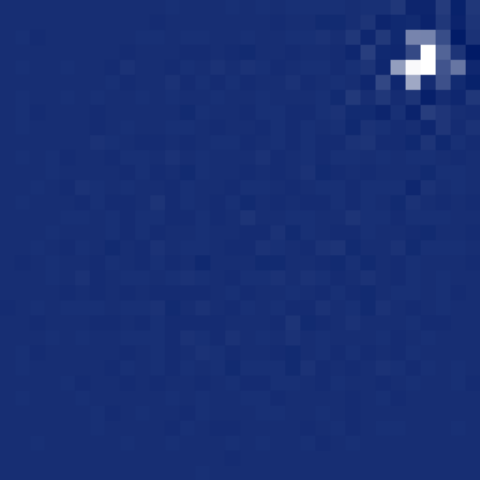}
		\caption{$\{2,10^2\}$}
		\label{subfig:augfalse2}
	\end{subfigure}\hfill
	\begin{subfigure}{0.12\textwidth}
		\includegraphics[width=\linewidth]{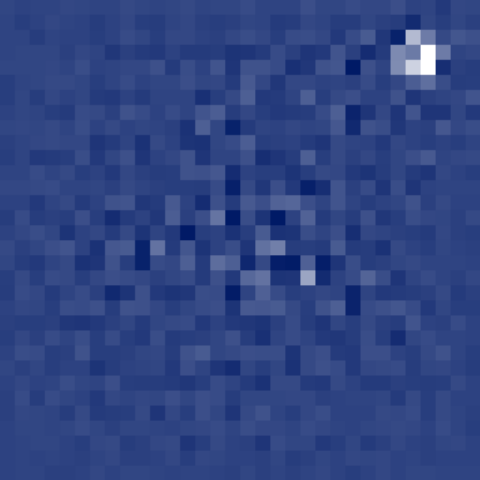}
		\caption{$\{2,10^3\}$}
		\label{subfig:augfalse3}
	\end{subfigure}\hfill
	\begin{subfigure}{0.12\textwidth}
		\includegraphics[width=\linewidth]{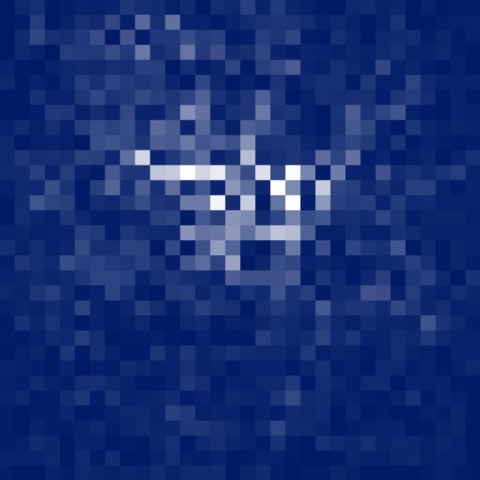}
		\caption{$\{1,10^3\}$}
	\end{subfigure}\hfill
	\begin{subfigure}{0.12\textwidth}
		\includegraphics[width=\linewidth]{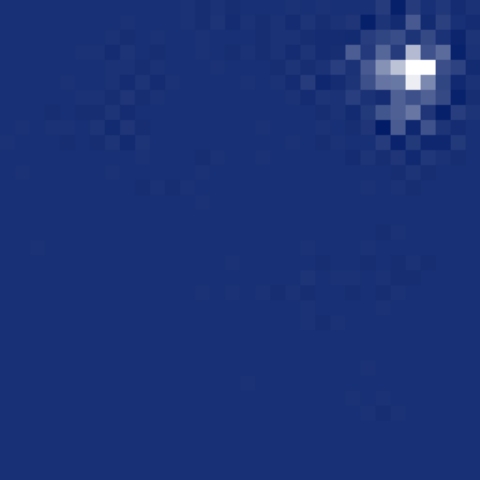}
		\caption{$\{4,0\}$}
	\end{subfigure}\hfill
	\begin{subfigure}{0.12\textwidth}
		\includegraphics[width=\linewidth]{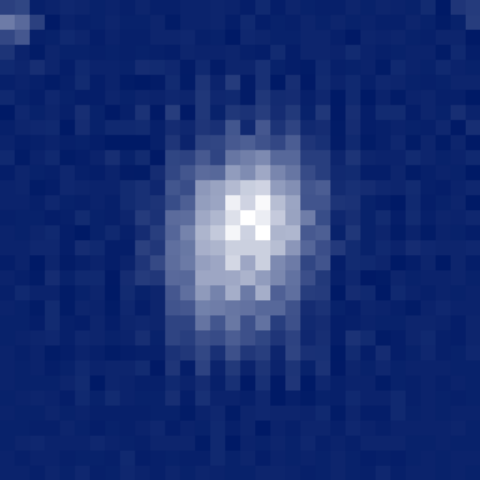}
		\caption{$\{8,10^2\}$}
		\label{subfig:augtrueshift1}
	\end{subfigure}\hfill
	\begin{subfigure}{0.12\textwidth}
		\includegraphics[width=\linewidth]{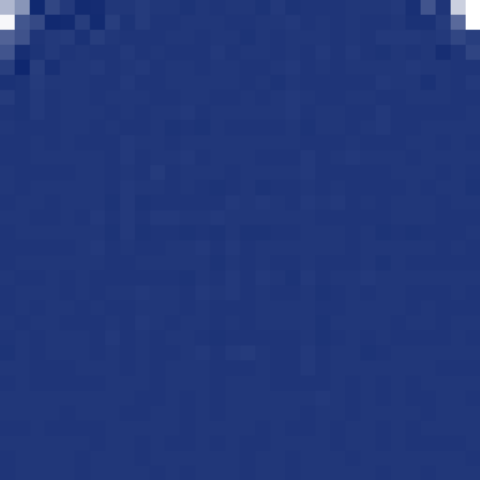}
		\caption{$\{12,10^3\}$}
		\label{subfig:augtrueshift2}
	\end{subfigure}\hfill
	\begin{subfigure}{0.12\textwidth}
		\includegraphics[width=\linewidth]{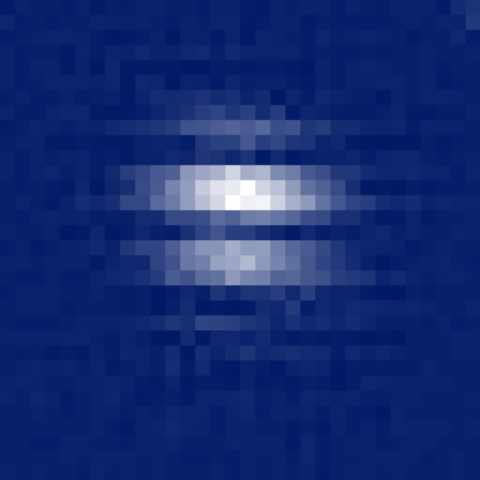}
		\caption{$\{8,10^3\}$}
	\end{subfigure}
	
	\caption{
		Top eigenvectors $v_1(H_x)$ across $(\theta \times 10^{-4}, \kappa)$.
		Left four: augmentation disabled.
		Right four: augmentation enabled. (d) and (h) represent cases of where no poison is visible.
	}
	\label{fig:eig_theta_kappa_8wide}
\end{figure*}

\begin{figure*}[!htbp]
	\centering
	\begin{subfigure}{0.49\linewidth}
		\centering
		\includegraphics[width=\linewidth]{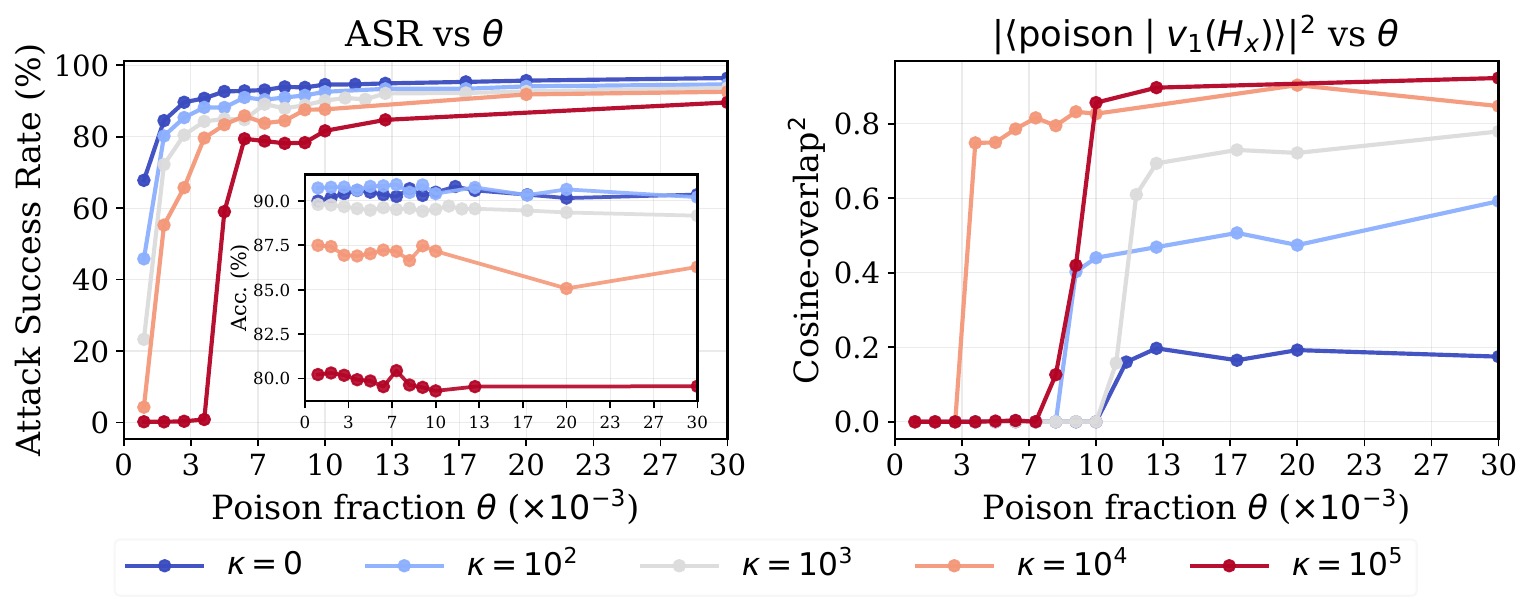}
		\caption{
			\textbf{CIFAR-10.}
		}
	\end{subfigure}
	\begin{subfigure}{0.49\linewidth}
		\centering
		\includegraphics[width=\linewidth]{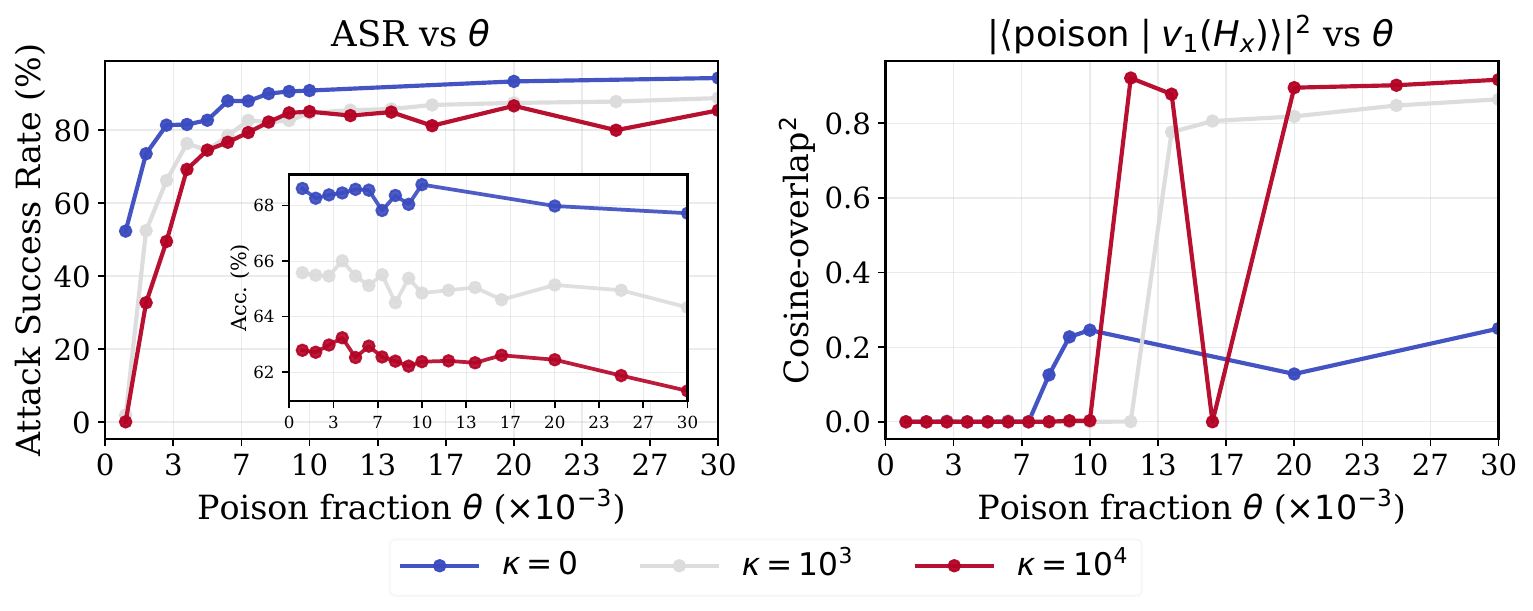}
		\caption{
			\textbf{CIFAR-100.}
		}
	\end{subfigure}\hfill
	
	\caption{
		Non-stochastic training. Attack success rate (left) and cosine-overlap$^{2}$ (right) versus poison fraction $\theta$
		for $\kappa \in \{0, >1000\}$ (augmentation disabled, 90 epochs).
		Inset: clean test accuracy.
		All axes use $10^{-3}$ scaling for poison fraction $\theta$.
		The shared legend indicates gradient regularisation strength $\kappa$.
	}
	\label{fig:combined_asr_overlap_nonstochastic_both}
\end{figure*}
\subsection{Stochastic Rank-1 Additive Poison}
\label{subsec:poisorig}
We adapt the small cross from  \citet{gu2017badnets} into an L shape and use method A for poisoning. Figure \ref{fig:combined_asr_insetacc_overlap_both} shows the attack success ratio (ASR) and top eigenvector cosine overlap with varying poison fraction $\theta$ for various regularisation strengths $\kappa$ for CIFAR-$10$($100$). We summarise the key findings specifically for CIFAR-$100$ in Figure \ref{fig:square_combined_clean_asr_new}. 
\begin{figure}[!htbp]
	\centering
	\begin{minipage}{0.49\linewidth}
		\begin{subfigure}{\linewidth}
			\centering
			\includegraphics[width=\linewidth]{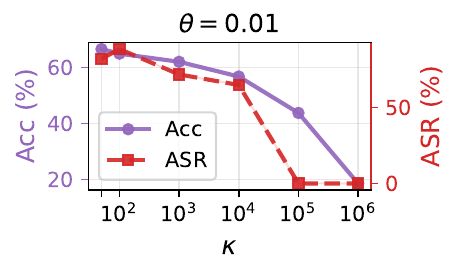}
			\caption{Accuracy vs ASR}
			\label{subfig:c100}
		\end{subfigure}
	\end{minipage}
	\hfill
	\begin{minipage}{0.49\linewidth}
		\begin{subfigure}{\linewidth}
			\centering
			\includegraphics[width=\linewidth]{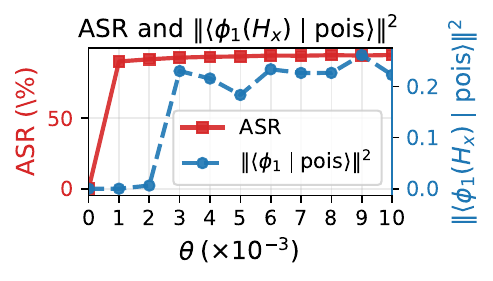}
			\caption{Delay in Poison Visibility}
			\label{subfig:c1002}
		\end{subfigure}
	\end{minipage}
	
	\caption{Key experimental findings for Deep Resnets on CIFAR-$100$. There is a clear trade off between improving the robustness to data poisoning and accuracy and there is a gap between poisoning detectability and poison efficacy in the attackers favour.}
	\label{fig:square_combined_clean_asr_new}
\end{figure}
ASR increases/decreases with $\theta,\kappa$ and there is a lag between ASR effectiveness and cosine overlap becoming non-trivial. Increasing $\kappa$ always reduces test set performance. In section \ref{subsec:poisnonstochastic} we investigate the late stage collapse in the overlap of the poison with $v_{1}(H_{x}$) but not the corresponding poison efficacy.

Figure \ref{fig:combined_asr_cleanacc_vs_theta_both} shows that data augmentation significantly improves the position of the defender.  \textit{once regularisation is employed} the poison fraction $\theta$ must be increased significantly to achieve an effective poison. As shown in Figures \ref{subfig:augtrueshift1}, \ref{subfig:augtrueshift2}, beyond a certain $\theta$ the top eigenvector of the input Hessian does display the poison trigger, but not necassarily in the same spot. This is in contrast to the no augmentation case, where as evidenced by Figures \ref{subfig:augfalse1}, \ref{subfig:augfalse2}, \ref{subfig:augfalse3}, (a noised version of) the poison appears in the identical location.

\begin{table}[h]
	\centering
	\caption{Results at $\theta=0.02$, Aug=True. All values are percentages (2 d.p.). E denotes epoch count.}
	\label{tab:kappa_1e4_3e4_1e5}
	\begin{tabular}{lccccccc}
		\toprule
		& \multicolumn{2}{c}{$\kappa=10^4$} 
		& \multicolumn{2}{c}{$\kappa=3\times 10^4$}
		& \multicolumn{2}{c}{$\kappa=10^5$} \\
		\cmidrule(lr){2-3}
		\cmidrule(lr){4-5}
		\cmidrule(lr){6-7}
		E & Acc & ASR & Acc & ASR & Acc & ASR \\
		\midrule
		90  & 80.85 & 82.10 & 79.18 & 71.96 & 75.17 & 6.63 \\
		450 & 84.19 & 88.18 & 82.28 & 81.99 & 79.05 & 70.68 \\
		\bottomrule
	\end{tabular}
\end{table}
We also consider, as shown in Table \ref{tab:kappa_1e4_3e4_1e5}, whether increased training (increasing the total number of epochs from $90 \rightarrow 450$) can close the efficacy gap from this form of regularisation. We find that whilst extra training does increase the training accuracy, we find a corresponding increase in poison efficacy also. Specifically we note that for $\theta=0.02, \kappa=10^5$, for a small ($\approx 2\%$) improvement in accuracy a near ineffective poison is boosted to nearly $70\%$ efficacy. We do see evidence of improvement of the Pareto frontier.
%

\subsection{Stochastic Warp Poison}
\label{subsec:warp}
We implement the imperceptible warp \citep{nguyen2021wanet} with strength $\phi=0.02$, the rest as in Section \ref{subsec:poisorig}. We summarise the findings in Figure \ref{fig:warp_combined_clean_asr}. Regularising with greater intensity, performance degrades but less than poison efficacy. We also visualise the top eigenvector of $\nabla^{2}_{x}L$ in Figure \ref{fig:eig_warp}, which displays a grid pattern even for small $\theta$, lost once the poison is no longer effective. Anecotally, such a method of poisoning which is harder to detect by the eye alone is more easily detected by our spectral methods.
\begin{figure}[!htbp]
	\centering
	\begin{minipage}{0.49\linewidth}
		\begin{subfigure}{\linewidth}
			\centering
			\includegraphics[width=\linewidth]{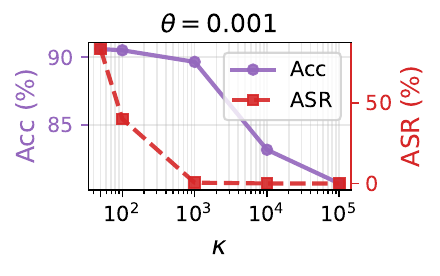}
			\caption{Poison fraction $=0.001$}
			\label{subfig:warp_combined_0p001}
		\end{subfigure}
	\end{minipage}
	\hfill
	\begin{minipage}{0.49\linewidth}
		\begin{subfigure}{\linewidth}
			\centering
			\includegraphics[width=\linewidth]{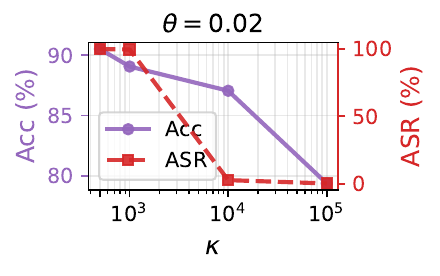}
			\caption{Poison fraction $=0.02$}
			\label{subfig:warp_combined_0p02}
		\end{subfigure}
	\end{minipage}
	\caption{Clean accuracy and ASR comparison for warp-poisoned models at two poison fractions ($0.001$ and $0.02$) Where $\kappa=10$ corresponds to $\kappa=0$.}
	\label{fig:warp_combined_clean_asr}
\vspace{-10pt}
\end{figure}
\begin{figure}[!htbp]
	\centering
	\begin{subfigure}{0.24\linewidth}
		\IfFileExists{eigenvector_panels_warp/warp_0.001_k0_augFalse.pdf}{
			\includegraphics[width=\linewidth]{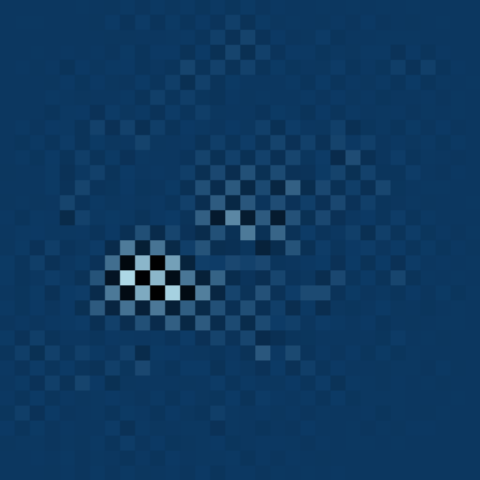}
		}{\fbox{\rule{0pt}{2cm}\rule{2cm}{0pt}Missing}}
		\caption{$\kappa{=}0$}
	\end{subfigure}\hfill
	\begin{subfigure}{0.24\linewidth}
		\IfFileExists{eigenvector_panels_warp/warp_0.001_k100.0_augFalse.pdf}{
			\includegraphics[width=\linewidth]{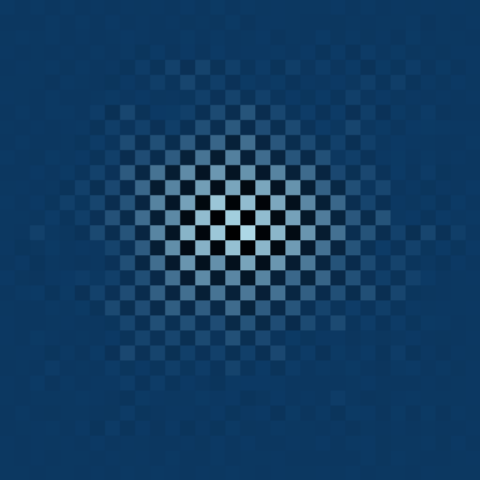}
		}{\fbox{\rule{0pt}{2cm}\rule{2cm}{0pt}Missing}}
		\caption{$\kappa{=}10^2$}
	\end{subfigure}\hfill
	\begin{subfigure}{0.24\linewidth}
		\IfFileExists{eigenvector_panels_warp/warp_0.001_k1000.0_augFalse.pdf}{
			\includegraphics[width=\linewidth]{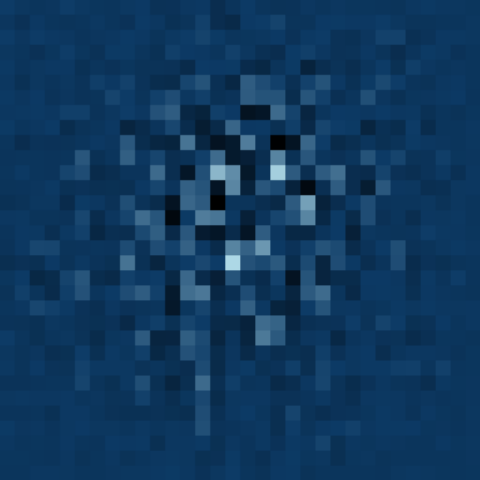}
		}{\fbox{\rule{0pt}{2cm}\rule{2cm}{0pt}Missing}}
		\caption{$\kappa{=}10^3$}
	\end{subfigure}\hfill
	\begin{subfigure}{0.24\linewidth}
		\IfFileExists{eigenvector_panels_warp/warp_0.001_k100000.0_augFalse.pdf}{
			\includegraphics[width=\linewidth]{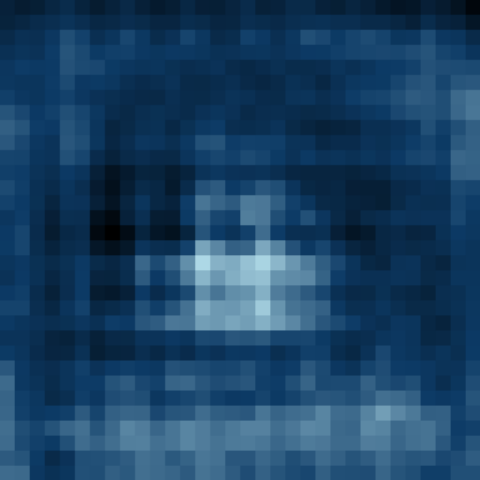}
		}{\fbox{\rule{0pt}{2cm}\rule{2cm}{0pt}Missing}}
		\caption{$\kappa{=}10^5$}
	\end{subfigure}
	\caption{Top eigenvectors $v_1(H_x)$ for $\theta{=}0.001$ warp poison, augmentation disabled,increasing $\kappa$.}
	\label{fig:eig_warp}
\end{figure}
\subsection{Deterministic Rank-1 Additive Poison}
\label{subsec:poisnonstochastic}
We find as shown in Figure \ref{fig:combined_asr_overlap_nonstochastic_both} similar qualitative results to the stochastic case. Attack success ratio is a little worse, we see less consistent ordering in onset of spectral markers as a function of $\kappa$ (no longer monatonic). We investigate the cause of the collapse of top eigenvector and poison overlap for the determinstic case in CIFAR-$100$ by positing, whether the poison might be rotating between the various eigenvectors of the Hessian. We plot $\textit{argmax}_{k}|<v_{k}(H_{x})|poison>|^{2}$ in Figure \ref{fig:combined_asr_insetacc_overlap_c100_nonstochastic_correct overlap} for which we find monotonic increase in $\theta$. We display the max overlap indices for a select set of indices in Table \ref{tab:clean_only_max_overlaps}. This confirms the poison rotation throughout the eigenvectors hypothesis.
\begin{table}[h]
	\centering
	\small
	\begin{tabular}{ccccccccc}
		\toprule
		$\kappa$ & $\theta$ & $max\ overlap$ & $max\ overlap\ index$\\
		\midrule
		0 & 0.01 & 0.4956 & 0 \\
		& 0.02 & 0.3765 & 1 \\
		& 0.03 & 0.4999 & 0 \\
		\midrule
		$10^4$ & 0.008 & 0.6666 & 2 \\
		& 0.009 & 0.7904 & 1 \\
		& 0.01 & 0.9363 & 1 \\
		& 0.02 & 0.9463 & 0 \\
		\bottomrule
	\end{tabular}
	\vspace{5pt}
	\caption{CIFAR-$10$ non stochastic poison overlap with poison fraction $\theta$ and eigenvector index of largest overlap}
	\label{tab:clean_only_max_overlaps}
\end{table}


\begin{figure}[!h]
	\centering
	\includegraphics[width=0.95\linewidth]{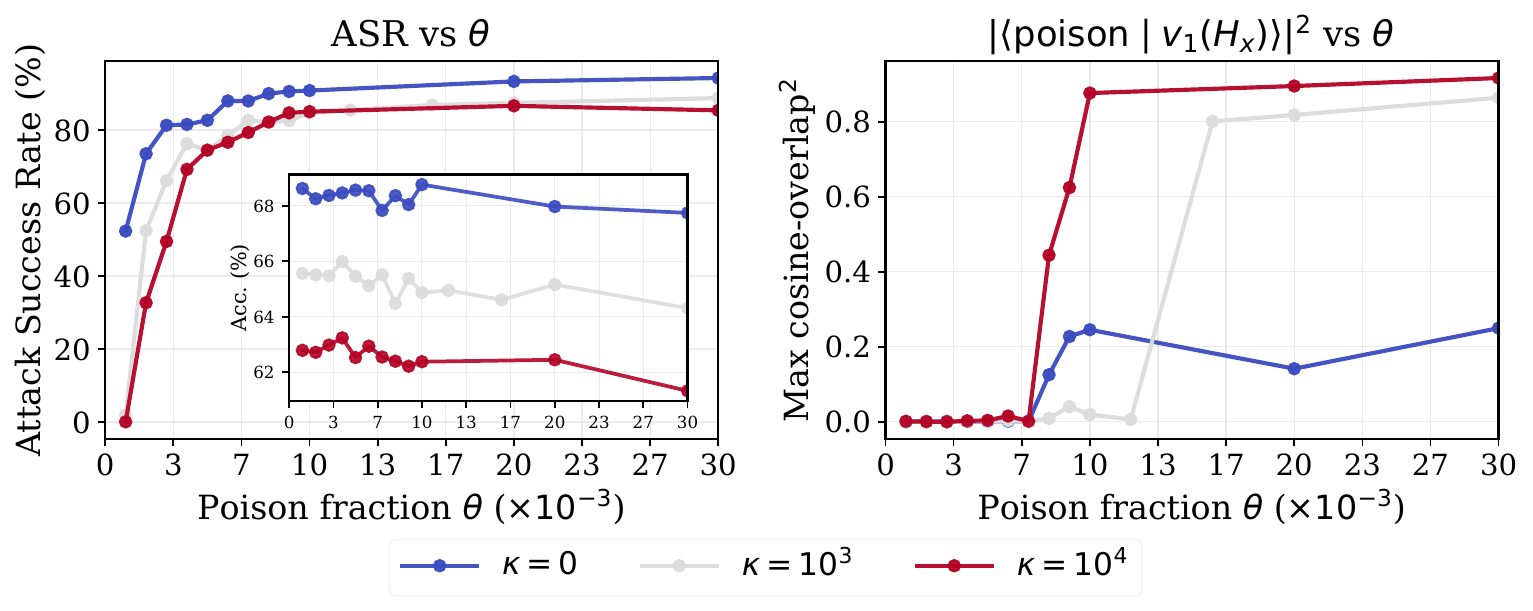}
	\caption{
		CIFAR-$100$ Combined view of attack success rate (left) and max-cosine-overlap$^{2}$ (right) versus poison fraction $\theta$ for $\kappa \in \{0, >1000\}$, (augmentation disabled, 90 epochs). 
		The inset in the left panel shows the corresponding clean test accuracy (bottom right). 
		All axes use the $10^{-3}$ scaling for poison fraction, and the shared legend below indicates the gradient regularisation $\kappa$ values.
	}
	\label{fig:combined_asr_insetacc_overlap_c100_nonstochastic_correct overlap}
\end{figure}

\paragraph{Reproducibility.}
All hyperparameters, poisoning procedures, and Hessian computations are described in detail in the appendix.

\section{Conclusion}
\begin{itemize}
	\item \textbf{Theoretical takeaway.}
	$|\nabla_x L|^2$ regularisation provably weakens data-fitting capacity while protecting against data poisoning, revealing an unavoidable safety--efficacy trade-off.
	
	\item \textbf{Practical takeaway.}
	Increased training combined with $|\nabla_x L|^2$ regularisation and data augmentation can yield models that are both highly accurate and robust to data poisoning.
\end{itemize}

In this paper we developed a unified theoretical and empirical framework for understanding data-poisoning and backdoor attacks through the geometry of the loss landscape in input space. Using kernel ridge regression as an exact model of wide neural networks, we showed that clustered dirty-label poisons induce a rank-one spike in the input Hessian whose magnitude scales quadratically with attack efficacy. Crucially, for nonlinear kernels we identified a near-clone regime in which poisoning remains order-one effective while the induced input curvature vanishes, rendering such attacks provably invisible to spectral detection methods. We evidence the near-clone regime experimentally, along with the implication of our theoretical results in deep neural networks trained under standard cross entropy loss.

We further analysed input-gradient regularisation as a defence mechanism and proved that it contracts poison-aligned Fisher and Hessian eigenmodes under gradient flow, while necessarily reducing data-fitting capacity. For exponential kernels, this regularisation admits a precise interpretation as anisotropic high-frequency suppression via an increased effective length scale, which shrinks the regime in which near-clone poisons can operate. We demonstrate consistent safety–efficacy trade-offs and further show that gradient regularisation and data augmentation act synergistically to produce the first known neural networks which are immune to data poisoning. Our Hessian analysis and experiments also give a novel method to detect strongly poisoned networks.

Overall, our results clarify when backdoors are fundamentally undetectable by spectral methods, why input-space curvature is a more informative diagnostic than weight-space analysis, and why no defence can eliminate poisoning without sacrificing expressive power. By reframing data poisoning as a geometric phenomenon in input space, this work provides a tractable foundation for analysing both attacks and defences in modern overparameterised models.

\section{Broader Impact}
This work contributes to the understanding of security risks in machine learning systems, particularly in safety-critical deployments where training data may be partially compromised. By characterising regimes in which backdoor attacks are inherently undetectable by spectral or curvature-based methods, our results expose fundamental limitations of existing defences and caution against over-reliance on post-hoc detection techniques. At the same time, our analysis of input-gradient regularisation provides principled guidance for mitigating poisoning, while making explicit the unavoidable trade-off between robustness and data-fitting capacity.

As with most research on adversarial machine learning, the insights presented here could be misused by adversaries to design more effective and less detectable poisoning attacks, for example by exploiting near-clone regimes in feature space. We mitigate this risk by focusing on general mechanisms rather than attack recipes, and by pairing all attack analysis with corresponding defensive implications and limitations. We believe that clearly articulating the fundamental constraints of defences ultimately strengthens real-world security by enabling practitioners to make informed, risk-aware design choices.

More broadly, this work encourages a shift from heuristic defences toward geometry-aware analyses of model sensitivity, which may inform the development of safer training procedures for high-stakes applications in areas such as healthcare, finance, and autonomous systems.

\bibliographystyle{icml2026}
\bibliography{bib}

@inproceedings{jagielski2018manipulating,
	title     = {Manipulating Machine Learning: Poisoning Attacks and Countermeasures for Regression Learning},
	author    = {Jagielski, Matthew and Oprea, Alina and Biggio, Battista and Liu, Chang and Nita-Rotaru, Cristina and Li, Bo},
	booktitle = {IEEE Symposium on Security and Privacy},
	pages     = {19--35},
	year      = {2018}
}

@inproceedings{liu2017robustlinreg,
	title     = {Robust Linear Regression Against Training Data Poisoning},
	author    = {Liu, Chang and Li, Bo and Vorobeychik, Yevgeniy and Oprea, Alina},
	booktitle = {Proceedings of the 10th ACM Workshop on Artificial Intelligence and Security (AISec)},
	pages     = {91--102},
	year      = {2017}
}

@inproceedings{muller2020regpoison,
	title     = {Data Poisoning Attacks on Regression Learning and Corresponding Defenses},
	author    = {M{\"u}ller, Nicolas Michael and Kowatsch, Daniel and B{\"o}ttinger, Konstantin},
	booktitle = {2020 IEEE 25th Pacific Rim International Symposium on Dependable Computing (PRDC)},
	pages     = {80--89},
	year      = {2020}
}

@inproceedings{gardner2018gpytorch,
	title     = {GPyTorch: Blackbox Matrix--Matrix Gaussian Process Inference with {GPU} Acceleration},
	author    = {Gardner, Jacob R. and Pleiss, Geoff and Bindel, David and Weinberger, Kilian Q. and Wilson, Andrew Gordon},
	booktitle = {Advances in Neural Information Processing Systems (NeurIPS)},
	year      = {2018}
}

@article{lanczos1950iteration,
	title   = {An Iteration Method for the Solution of the Eigenvalue Problem of Linear Differential and Integral Operators},
	author  = {Lanczos, Cornelius},
	journal = {Journal of Research of the National Bureau of Standards},
	volume  = {45},
	pages   = {255--282},
	year    = {1950}
}

@article{zhao2024robustnp,
	title   = {Robust Nonparametric Regression under Poisoning Attack},
	author  = {Zhao, Puning and Wan, Zhiguo},
	journal = {Proceedings of the AAAI Conference on Artificial Intelligence},
	year    = {2024},
	note    = {Also arXiv:2305.16771}
}

@inproceedings{zhu2022arks,
	title     = {Adversarially Robust Kernel Smoothing},
	author    = {Zhu, Jia-Jie and Kouridi, Christina and Nemmour, Yassine and Sch{\"o}lkopf, Bernhard},
	booktitle = {Proceedings of the 25th International Conference on Artificial Intelligence and Statistics (AISTATS)},
	series    = {PMLR},
	volume    = {151},
	pages     = {4972--4994},
	year      = {2022}
}

@article{allerbo2025fastrobustkrr,
	title   = {Fast Robust Kernel Regression through Sign Gradient Descent with Early Stopping},
	author  = {Allerbo, Oskar},
	journal = {Electronic Journal of Statistics},
	volume  = {19},
	number  = {1},
	pages   = {1231--1285},
	year    = {2025}
}

@inproceedings{deng2020aif,
	title     = {Interpreting Robust Optimization via Adversarial Influence Functions},
	author    = {Deng, Zhun and Dwork, Cynthia and Wang, Jialiang and Zhang, Linjun},
	booktitle = {Proceedings of the 37th International Conference on Machine Learning (ICML)},
	series    = {PMLR},
	volume    = {119},
	pages     = {2452--2462},
	year      = {2020}
}

@article{ribeiro2025kaf,
	title   = {Kernel Learning with Adversarial Features: Numerical Efficiency and Adaptive Regularization},
	author  = {Ribeiro, Ant{\^o}nio H. and V{\"a}vinggren, David and Zachariah, Dave and Sch{\"o}n, Thomas B. and Bach, Francis},
	journal = {Advances in Neural Information Processing Systems},
	year    = {2025},
	note    = {NeurIPS 2025, to appear}
}

@inproceedings{jacot2018ntk,
	title     = {Neural Tangent Kernel: Convergence and Generalization in Neural Networks},
	author    = {Jacot, Arthur and Gabriel, Franck and Hongler, Cl{\'e}ment},
	booktitle = {Advances in Neural Information Processing Systems},
	volume    = {31},
	year      = {2018}
}

@inproceedings{arora2019exact,
	title     = {On Exact Computation with an Infinitely Wide Neural Net},
	author    = {Arora, Sanjeev and Du, Simon S. and Hu, Wei and Li, Zhiyuan and Salakhutdinov, Ruslan and Wang, Ruosong},
	booktitle = {Advances in Neural Information Processing Systems},
	volume    = {32},
	pages     = {8139--8148},
	year      = {2019}
}

@inproceedings{wang2021robustlearning,
	title     = {Robust Learning for Data Poisoning Attacks},
	author    = {Wang, Yunjuan and Mianjy, Poorya and Arora, Raman},
	booktitle = {Proceedings of the 38th International Conference on Machine Learning (ICML)},
	series    = {PMLR},
	volume    = {139},
	pages     = {10859--10869},
	year      = {2021}
}

@article{karmakar2020depth2,
	title   = {Depth-2 Neural Networks Under a Data-Poisoning Attack},
	author  = {Karmakar, Sayar and Mukherjee, Anirbit and Papamarkou, Theodore},
	journal = {Neurocomputing},
	year    = {2023},
	note    = {Earlier version arXiv:2005.01699}
}

@article{geiping2021poison,
	title   = {What Doesn't Kill You Makes You Robust(er): How to Adversarially Train Against Data Poisoning},
	author  = {Geiping, Jonas and Fowl, Liam H. and Somepalli, Gowthami and Goldblum, Micah and Moeller, Michael and Goldstein, Tom},
	journal = {arXiv preprint arXiv:2102.13624},
	year    = {2021}
}

@inproceedings{liu2022friendly,
	title     = {Friendly Noise Against Adversarial Noise: A Powerful Defense Against Data Poisoning Attacks},
	author    = {Liu, Tian Yu and Yang, Yu and Mirzasoleiman, Baharan},
	booktitle = {Advances in Neural Information Processing Systems},
	volume    = {35},
	year      = {2022}
}

@inproceedings{wei2023shared,
	title     = {Shared Adversarial Unlearning: Backdoor Mitigation by Unlearning Shared Adversarial Examples},
	author    = {Wei, Shaokui and Zhang, Mingda and Zha, Hongyuan and Wu, Baoyuan},
	booktitle = {Advances in Neural Information Processing Systems},
	year      = {2023}
}

@inproceedings{bal2025adversarial,
	title     = {Adversarial Training for Defense Against Label Poisoning Attacks},
	author    = {Bal, Melis Ilayda and Cevher, Volkan and Muehlebach, Michael},
	booktitle = {International Conference on Learning Representations},
	year      = {2025}
}

@article{hallaji2023label,
	title   = {Label Noise Analysis Meets Adversarial Training: A Defense Against Label Poisoning in Federated Learning},
	author  = {Hallaji, Ehsan and Razavi-Far, Roozbeh and Saif, Mehrdad and Herrera-Viedma, Enrique},
	journal = {Knowledge-Based Systems},
	volume  = {266},
	pages   = {110384},
	year    = {2023},
	doi     = {10.1016/j.knosys.2023.110384}
}

@misc{granziol2025linearapproachdatapoisoning,
	title={A Linear Approach to Data Poisoning}, 
	author={Diego Granziol and Donald Flynn},
	year={2025},
	eprint={2505.15175},
	archivePrefix={arXiv},
	primaryClass={stat.ML},
	url={https://arxiv.org/abs/2505.15175}, 
}

@inproceedings{he2016deep,
	title      = {Deep Residual Learning for Image Recognition},
	author     = {He, Kaiming and Zhang, Xiangyu and Ren, Shaoqing and Sun, Jian},
	booktitle  = {Proceedings of the IEEE Conference on Computer Vision and Pattern Recognition},
	pages      = {770--778},
	year       = {2016}
}

@inproceedings{nguyen2021wanet,
	title        = {WaNet -- Imperceptible Warping-based Backdoor Attack},
	author       = {Nguyen, Anh and Tran, Toan and Tran, Anh},
	booktitle    = {International Conference on Learning Representations (ICLR)},
	year         = {2021},
	url          = {https://openreview.net/forum?id=rCr-savc8t},
	note         = {OpenReview ID: rCr-savc8t}
}

@inproceedings{Xiang2024BadChain,
	title     = {{BadChain}: Backdoor Chain-of-Thought Prompting for Large Language Models},
	author    = {Xiang, Zhen and Jiang, Fengqing and Xiong, Zidi and Ramasubramanian, Bhaskar and Poovendran, Radha and Li, Bo},
	booktitle = {International Conference on Learning Representations (ICLR)},
	year      = {2024}
}

@inproceedings{Li2024BadEdit,
	title     = {{BadEdit}: Backdooring Large Language Models by Model Editing},
	author    = {Li, Yanzhou and Li, Tianlin and Chen, Kangjie and Zhang, Jian and Liu, Shangqing and Wang, Wenhan and Zhang, Tianwei and Liu, Yang},
	booktitle = {International Conference on Learning Representations (ICLR)},
	year      = {2024}
}

@article{Shi2023BadGPT,
	title   = {{BadGPT}: Exploring Security Vulnerabilities of {ChatGPT} via Backdoor Attacks to {InstructGPT}},
	author  = {Shi, Jiawen and Liu, Yixin and Zhou, Pan and Sun, Lichao},
	journal = {arXiv preprint arXiv:2304.12298},
	year    = {2023}
}

@inproceedings{Cai2022BadPrompt,
	title     = {{BadPrompt}: Backdoor Attacks on Continuous Prompts},
	author    = {Cai, Xiangrui and Xu, Haidong and Xu, Sihan and Zhang, Ying and Yuan, Xiaojie},
	booktitle = {Advances in Neural Information Processing Systems 35 (NeurIPS 2022)},
	year      = {2022}
}

@article{Hubinger2024Sleeper,
	title   = {Sleeper Agents: Training Deceptive {LLMs} that Persist through Safety Training},
	author  = {Hubinger, Evan and Denison, Carson and Mu, Jesse and Lambert, Mike and Tong, Meg and MacDiarmid, Monte and Lanham, Tamera and Ziegler, Daniel M. and Maxwell, Tim and Cheng, Newton and others},
	journal = {arXiv preprint arXiv:2401.05566},
	year    = {2024}
}

@article{xu2024shadowcast,
	title={Shadowcast: Stealthy data poisoning attacks against vision-language models},
	author={Xu, Yuancheng and Yao, Jiarui and Shu, Manli and Sun, Yanchao and Wu, Zichu and Yu, Ning and Goldstein, Tom and Huang, Furong},
	journal={arXiv preprint arXiv:2402.06659},
	year={2024}
}

@article{wang2024stronger,
	title={The stronger the diffusion model, the easier the backdoor: Data poisoning to induce copyright breaches without adjusting finetuning pipeline},
	author={Wang, Haonan and Shen, Qianli and Tong, Yao and Zhang, Yang and Kawaguchi, Kenji},
	journal={arXiv preprint arXiv:2401.04136},
	year={2024}
}

@article{pan2024trojan,
	title={From Trojan Horses to Castle Walls: Unveiling Bilateral Data Poisoning Effects in Diffusion Models},
	author={Pan, Zhuoshi and Yao, Yuguang and Liu, Gaowen and Shen, Bingquan and Zhao, H Vicky and Kompella, Ramana and Liu, Sijia},
	journal={Advances in Neural Information Processing Systems},
	volume={37},
	pages={82265--82295},
	year={2024}
}

@article{he2022indiscriminate,
	title={Indiscriminate poisoning attacks on unsupervised contrastive learning},
	author={He, Hao and Zha, Kaiwen and Katabi, Dina},
	journal={arXiv preprint arXiv:2202.11202},
	year={2022}
}

@inproceedings{papernot2018sok,
  title     = {{SoK}: Security and Privacy in Machine Learning},
  author    = {Papernot, Nicolas and McDaniel, Patrick and Sinha, Arunesh and Wellman, Michael P},
  booktitle = {2018 IEEE European Symposium on Security and Privacy (EuroS\&P)},
  pages     = {399--414},
  year      = {2018}
}

@inproceedings{carlini2019secret,
  title     = {The Secret Sharer: Measuring Unintended Neural Network Memorization \& Extracting Secrets},
  author    = {Carlini, Nicholas and Liu, Chang and Erlingsson, {\'U}lfar and Fernandez, J{\'e}r{\^o}me and Song, Dawn},
  booktitle = {28th USENIX Security Symposium (USENIX Security)},
  pages     = {267--284},
  year      = {2019}
}

@misc{gu2017badnets,
  title        = {BadNets: Identifying Vulnerabilities in the Machine Learning Model Supply Chain},
  author       = {Gu, Tianyu and Dolan-Gavitt, Brendan and Garg, Siddharth},
  year         = {2017},
  archivePrefix= {arXiv},
  eprint       = {1708.06733},
  primaryClass = {cs.CR}
}

@inproceedings{carlini2022poisoning,
  title={Poisoning and Backdooring Contrastive Learning},
  author={Carlini, Nicholas and Terzis, Andreas},
  booktitle={Proceedings of the International Conference on Learning Representations (ICLR)},
  year={2022},
  url={https://arxiv.org/abs/2106.09667}
}

@article{papyan2020neuralcollapse,
	title   = {Prevalence of Neural Collapse during the Terminal Phase of Deep Learning Training},
	author  = {Papyan, Vardan and Han, X. Y. and Donoho, David L.},
	journal = {Proceedings of the National Academy of Sciences},
	volume  = {117},
	number  = {40},
	pages   = {24652--24663},
	year    = {2020},
	doi     = {10.1073/pnas.2015509117},
	url     = {https://arxiv.org/abs/2008.08186}
}

@article{lu2020nccrossentropy,
	title   = {Neural Collapse with Cross-Entropy Loss},
	author  = {Lu, Jianfeng and Steinerberger, Stefan},
	journal = {Applied and Computational Harmonic Analysis},
	year    = {2022},
	note    = {First posted as arXiv:2012.08465},
	url     = {https://arxiv.org/abs/2012.08465}
}

@misc{ji2021ulpm,
	title         = {How Gradient Descent Separates Data with Neural Collapse: A Layer-Peeled Perspective},
	author        = {Ji, Ziwei and Dud{\'i}k, Miroslav and Schapire, Robert E. and Telgarsky, Matus},
	year          = {2021},
	eprint        = {2110.02796},
	archivePrefix = {arXiv},
	primaryClass  = {cs.LG},
	url           = {https://arxiv.org/abs/2110.02796}
}

@article{hong2024jmlr,
	title   = {Neural Collapse for Unconstrained Feature Model under Cross-Entropy Loss with Imbalanced Data},
	author  = {Hong, Wanli and Ling, Shuyang},
	journal = {Journal of Machine Learning Research},
	volume  = {25},
	number  = {220},
	pages   = {1--48},
	year    = {2024},
	url     = {https://www.jmlr.org/papers/volume25/23-1215/23-1215.pdf}
}

@inproceedings{han2021mse,
	title     = {Neural Collapse Under MSE Loss: Proximity to and Dynamics on the Central Path},
	author    = {Han, X. Y. and Papyan, Vardan and Donoho, David L.},
	booktitle = {International Conference on Learning Representations (ICLR)},
	year      = {2022},
	note      = {arXiv:2106.02073},
	url       = {https://arxiv.org/abs/2106.02073}
}


\appendix
\section{Theoretical Proofs}
\label{app:proofs}

\subsection{Aggregate Poison Gain}

\begin{lemma}[Aggregate poison gain]
	\label{lem:sum-alpha}
	Under Assumption~\ref{asmp:cluster},
	\[
	\mathbf{1}^\top \boldsymbol{\alpha}_P
	=
	y_t\,\frac{m}{n\lambda+\kappa_0 m}
	=
	y_t\,S(m;\lambda).
	\]
\end{lemma}

\begin{proof}
	On the poison block, the regularised Gram matrix satisfies
	\begin{align}
		K_{PP}+n\lambda I_m
		&\approx \kappa_0\,\mathbf{1}\mathbf{1}^\top+n\lambda I_m .
	\end{align}
	Applying the Sherman--Morrison identity,
	\begin{align}
		(\kappa_0\mathbf{1}\mathbf{1}^\top+n\lambda I_m)^{-1}\mathbf{1}
		&=
		\frac{1}{n\lambda+\kappa_0 m}\mathbf{1}.
	\end{align}
	Multiplying by $y_t\mathbf{1}$ and summing entries yields the result.
\end{proof}


\subsection{Efficacy of a Cloned Poison Cluster}

\begin{theorem}[Efficacy of a cloned cluster]
	\label{thm:efficacy}
	Under Assumption~\ref{asmp:cluster}, the poison-induced score shift at a triggered test point $x_0$ is
	\[
	\Delta f(x_0)
	=
	k(x_0,\zeta)\,y_t\,\frac{m}{n\lambda+\kappa_0 m}.
	\]
	For $m\ll n\lambda/\kappa_0$, $\Delta f(x_0)=\Theta(m)$; as $m\to\infty$ it saturates.
\end{theorem}

\begin{proof}
	By the representer theorem,
	\begin{align}
		f(x_0)
		&= \sum_{i=1}^n \alpha_i k(x_0,x_i).
	\end{align}
	Under Assumption~\ref{asmp:cluster}(b), the poison contribution dominates:
	\begin{align}
		\Delta f(x_0)
		&\approx \sum_{i\in P}\alpha_i k(x_0,\zeta)
		= k(x_0,\zeta)\,\mathbf{1}^\top\boldsymbol{\alpha}_P .
	\end{align}
	Substituting Lemma~\ref{lem:sum-alpha} gives the claim.
\end{proof}


\subsection{Rank-1 Input-Hessian Spike}

\begin{theorem}[Rank-1 input-Hessian spike and spike–efficacy law]
	\label{thm:spike}
	Under Assumption~\ref{asmp:cluster},
	\begin{align}
		\|\nabla_x f(x_0)\|^2
		&=
		\|\nabla_x k(x_0,\zeta)\|^2\,S(m;\lambda)^2,
		\label{eq:spike-grad}
		\\
		\lambda_{\max}\!\left(\nabla_x^2 L(x_0)\right)
		&=
		\frac{\|\nabla_x k(x_0,\zeta)\|^2}{k(x_0,\zeta)^2}
		\bigl(\Delta f(x_0)\bigr)^2 .
		\label{eq:spike-law}
	\end{align}
	Thus the curvature spike scales quadratically in efficacy.
\end{theorem}

\begin{proof}
	From $\nabla_x f(x)=\sum_i\alpha_i\nabla_x k(x,x_i)$,
	\begin{align}
		\nabla_x f(x_0)
		&\approx (\mathbf{1}^\top\boldsymbol{\alpha}_P)\,\nabla_x k(x_0,\zeta)
		= y_t S(m;\lambda)\,\nabla_x k(x_0,\zeta),
	\end{align}
	giving \eqref{eq:spike-grad}.  
	
	For squared loss,
	\(
	\nabla_x^2 L=\nabla_x f\,\nabla_x f^\top+(f-y)\nabla_x^2 f
	\),
	whose leading rank-1 term has eigenvalue $\|\nabla_x f\|^2$.
	Eliminating $S(m;\lambda)$ using $\Delta f(x_0)=k(x_0,\zeta)y_t S(m;\lambda)$ yields
	\eqref{eq:spike-law}.
\end{proof}


\subsection{Near-Clone Regime and the Twilight Zone}

\begin{proposition}[Near-clone regime]
	\label{prop:nearclone}
	Let $k(x,x')=\exp(-\|x-x'\|^2/(2\ell^2))$ and $r=\|x_0-\zeta\|\ll\ell$.
	Then
	\begin{align}
		\Delta f(x_0)
		&= y_t S(m;\lambda)\bigl(1+O(r^2/\ell^2)\bigr),
		\\
		\lambda_{\max}\!\left(\nabla_x^2 L(x_0)\right)
		&= S(m;\lambda)^2\,\frac{r^2}{\ell^4}+O(r^4/\ell^6).
	\end{align}
\end{proposition}

\begin{proof}
	A Taylor expansion gives
	\begin{align}
		k(x_0,\zeta)
		&= 1-\frac{r^2}{2\ell^2}+O(r^4/\ell^4),
		\\
		\|\nabla_x k(x_0,\zeta)\|^2
		&= \frac{r^2}{\ell^4}+O(r^4/\ell^6).
	\end{align}
	Substituting into Theorems~\ref{thm:efficacy} and~\ref{thm:spike} yields the result.
\end{proof}

\noindent
\textbf{Interpretation.}
In this regime, poison efficacy is order-one while curvature vanishes quadratically, producing effective but spectrally hidden backdoors.


\subsection{Capacity Loss under Gradient / Length-Scale Regularisation}

\begin{theorem}[Monotone loss of data-fitting capacity]
	\label{thm:capacity}
	Let $K_\ell$ be a stationary kernel Gram matrix with length-scale $\ell$.
	Then the effective degrees of freedom
	\[
	\mathrm{df}(\ell)
	=
	\mathrm{tr}\!\left[K_\ell(K_\ell+\lambda I)^{-1}\right]
	\]
	is strictly decreasing in $\ell$, and the training residual norm is strictly increasing.
\end{theorem}

\begin{proof}
	Let $K_\ell=U\Lambda_\ell U^\top$.
	Predictions satisfy
	\begin{align}
		\hat y
		&= U\,\mathrm{diag}\!\left(
		\frac{\sigma_j(\ell)}{\sigma_j(\ell)+\lambda}
		\right)U^\top y .
	\end{align}
	Thus
	\begin{align}
		\mathrm{df}(\ell)
		&= \sum_j \frac{\sigma_j(\ell)}{\sigma_j(\ell)+\lambda}.
	\end{align}
	For stationary kernels, increasing $\ell$ shifts spectral mass toward low frequencies while preserving trace.
	Since $\partial s_j/\partial\sigma_j>0$ and $\sum_j d\sigma_j/d\ell=0$, the weighted sum decreases.
	Each smoothing factor decreases, so the residual norm increases monotonically.
\end{proof}


\subsection{Fisher Contraction under Gradient Flow}

\begin{theorem}[Fisher contraction and directional decay]
	\label{thm:fisherflow}
	Let $g_w(x)=\nabla_x L(w;x)$ and $F(w)=\mathbb{E}[g_w g_w^\top]$.
	Under gradient flow on
	\[
	\mathcal{J}(w)
	=
	\mathbb{E}L(w;x)
	+
	\frac{\kappa}{2}\mathbb{E}\|g_w(x)\|^2,
	\]
	the Fisher matrix contracts in Loewner order.
	Moreover, if $v^\top A(w,x)v\ge\alpha>0$ with
	$A(w,x)=\partial_w g_w(x)\partial_w g_w(x)^\top$, then
	\[
	v^\top F(t)v \le e^{-2\kappa\alpha t}v^\top F(0)v .
	\]
\end{theorem}

\begin{proof}
	The penalty contribution to the gradient evolution is
	\begin{align}
		\dot g_w(x)\big|_{\mathrm{pen}}
		&\approx -\kappa\,A(w,x)\,g_w(x).
	\end{align}
	Using $\tfrac{d}{dt}(gg^\top)=\dot g g^\top+g\dot g^\top$ and averaging,
	\begin{align}
		\dot F
		&= -\kappa(AF+FA),
	\end{align}
	which is negative semidefinite.
	For unit $v$,
	\begin{align}
		\frac{d}{dt}v^\top F v
		&= -2\kappa\,\mathbb{E}\big[(v^\top g)(v^\top A g)\big]
		\le -2\kappa\alpha\,v^\top F v,
	\end{align}
	yielding the exponential bound.
\end{proof}


\section{Neural Collapse Results}
\label{app:ncollapse}
\begin{table}[H]
	\centering
	\tiny
	\begin{tabular}{ccccccccc}
		\toprule
		$c$ &
		$\|\mu_c\|$ &
		$\|\sigma_c\|$ &
		$\|\mu_c^{\star}\|$ &
		$\|\sigma_c^{\star}\|$ &
		$\frac{\|\sigma_c^{\star}\|}{\|\sigma_c\|}$ &
		$d(\mu_c)$ &
		$d(\mu_c^{\star})$ &
		$\frac{d^{\star}}{d}$ \\
		\midrule
		0 & 5.2 & 2.2 & 6.7 & 0.9 & 0.4 & 0.0 & 1.9 & --   \\
		1 & 5.4 & 1.8 & 5.9 & 1.4 & 0.8 & 5.9 & 1.6 & 0.3 \\
		2 & 5.4 & 2.5 & 5.8 & 1.1 & 0.4 & 5.1 & 1.1 & 0.2 \\
		3 & 5.0 & 2.6 & 5.9 & 1.2 & 0.4 & 5.2 & 1.4 & 0.3 \\
		4 & 5.4 & 2.3 & 5.6 & 1.0 & 0.4 & 5.9 & 1.1 & 0.2 \\
		5 & 5.1 & 2.4 & 6.0 & 1.0 & 0.4 & 5.9 & 1.4 & 0.2 \\
		6 & 5.3 & 2.0 & 5.5 & 1.0 & 0.5 & 6.2 & 1.2 & 0.2 \\
		7 & 5.1 & 2.0 & 5.9 & 1.2 & 0.6 & 5.6 & 1.4 & 0.3 \\
		8 & 5.2 & 1.8 & 6.2 & 1.0 & 0.6 & 5.2 & 1.6 & 0.3 \\
		9 & 5.5 & 1.9 & 6.0 & 1.7 & 0.9 & 5.6 & 1.8 & 0.3 \\
		\bottomrule
	\end{tabular}
	\caption{CIFAR-10 pre-\texttt{fc} feature statistics for clean vs.\ triggered poisons
		at $\theta = 0.01$. For each class $c$, $\mu_c$ and $\sigma_c$ denote the mean and
		per-dimension standard deviation of clean features, and $\mu_c^{\star}, \sigma_c^{\star}$
		their poisoned counterparts. Distances $d(\mu_c,\mu_0)$ and $d(\mu_c^{\star},\mu_0)$
		are measured to the class-0 clean mean $\mu_0$. Poisons are consistently tighter
		($\|\sigma_c^{\star}\|/\|\sigma_c\| < 1$) and lie much closer to the target
		class-0 manifold ($d^{\star}/d \ll 1$ for $c \neq 0$), indicating that triggered
		examples from any class are mapped into a compact region near class 0 in feature
		space rather than remaining near their original class manifold.}
	\label{tab:cifar_feature_cluster_stats}
\end{table}

\section{Poisoning Proofs}
\label{app:poison}

\begin{proof}[Proof of gain identity]
	On the poison block, $K_{PP}+cI_m\approx\kappa_0\mathbf{1}\mathbf{1}^\top+cI_m$.
	By Sherman--Morrison,
	\[
	(\kappa_0\mathbf{1}\mathbf{1}^\top+cI_m)^{-1}\mathbf{1}
	=\frac{1}{c+\kappa_0 m}\mathbf{1}.
	\]
	Multiplying by $y_t\mathbf{1}$ gives $\mathbf{1}^\top\alpha_P=y_t S(m;\lambda)$.
\end{proof}

\begin{proof}[Proof of spike--efficacy law]
	Under dominance at $x_0$,
	\[
	\nabla_x f(x_0)\approx(\mathbf{1}^\top\alpha_P)\nabla_x k(x_0,\zeta)
	=y_t S(m;\lambda)\nabla_x k(x_0,\zeta).
	\]
	Taking norms yields
	$\Lambda_{\mathrm{GN}}=\|\nabla_x k(x_0,\zeta)\|^2 S(m;\lambda)^2$.
	Substituting $\Delta f(x_0)=\kappa y_t S(m;\lambda)$ completes the proof.
\end{proof}

\begin{lemma}[Deep exponential kernel chain rule]
	Let $k(x,x')=\exp(-\|\phi(x)-\phi(x')\|^2/(2\ell^2))$. Then
	\[
	\nabla_x k(x_0,\zeta)
	=
	-\frac{\kappa}{\ell^2}J_\phi(x_0)^\top(\phi(x_0)-\phi(\zeta)),
	\]
	\[
	\|\nabla_x k(x_0,\zeta)\|^2
	=
	\frac{\kappa^2}{\ell^4}\|J_\phi(x_0)^\top(\phi(x_0)-\phi(\zeta))\|^2 .
	\]
\end{lemma}

\section{Safety--Efficacy Proofs}
\label{app:defense}

\begin{proof}[Proof of Cauchy--Schwarz bound]
	By Cauchy--Schwarz in $L_2(\mu\times\nu)$,
	\[
	|\mathbb{E}[\nabla_x L^\top\Delta]|
	\le
	\big(\mathbb{E}\|\nabla_x L\|^2\big)^{1/2}
	\big(\mathbb{E}\|\Delta\|^2\big)^{1/2}.
	\]
\end{proof}

\begin{proof}[Kernel length-scale and degrees of freedom]
	For stationary kernels, eigenvalues $\sigma_j(\ell)$ satisfy
	$\sum_j\sigma_j(\ell)=\mathrm{tr}(K_\ell)$ constant.
	Since $\sigma_j(\ell)$ shift mass toward low frequencies as $\ell$ increases,
	\[
	\frac{d}{d\ell}\sum_j\frac{\sigma_j(\ell)}{\sigma_j(\ell)+\lambda}<0,
	\]
	yielding monotone decrease of $\mathrm{df}(\ell)$ and monotone increase of the residual.
\end{proof}

\begin{proof}[Spectral shrinkage]
	By Plancherel,
	\[
	\int \|\nabla f(x)\|^2dx=\int \|\omega\|^2|\widehat f(\omega)|^2d\omega,
	\]
	and the RKHS norm satisfies
	$\|f\|_{\mathcal{H}_k}^2=\int |\widehat f(\omega)|^2/\widehat\kappa(\omega)d\omega$.
	Minimizing pointwise in $\omega$ yields
	\[
	\widehat f(\omega)=\frac{1}{1+\lambda/\widehat\kappa(\omega)+\eta\|\omega\|^2}\widehat y(\omega).
	\]
\end{proof}

\section{Implementation Details}

This chapter provides complete implementation details for reproducing our experiments on data poisoning and input-space Hessian analysis. We include the core training procedure with gradient regularization and the Hessian eigenspectrum computation pipeline.

\subsection{Training with Gradient Regularization}

Our training procedure implements gradient regularization to investigate its effect on poison memorization in the loss landscape. The key challenge is computing the input-space gradient penalty efficiently during training.

\paragraph{Gradient Regularization Penalty:}

The gradient regularization term penalizes large gradients with respect to the input:
\[
\mathcal{L}_{\text{total}} = \mathcal{L}_{\text{CE}}(f(x), y) + \kappa \|\nabla_x \mathcal{L}_{\text{CE}}(f(x), y)\|^2
\]

where $\kappa$ is the regularization strength. This requires computing gradients with respect to the input rather than model parameters.

\begin{lstlisting}[caption={Gradient Regularization Implementation}]
def compute_grad_reg_loss(model, inputs, targets, kappa):
	
	# Enable gradient computation w.r.t. inputs
	inputs.requires_grad_(True)
	
	# Forward pass
	outputs = model(inputs)
	loss = F.cross_entropy(outputs, targets)
	
	# Compute gradient w.r.t. input
	grad_inputs = torch.autograd.grad(
			outputs=loss,
			inputs=inputs,
			create_graph=True,  # Enable second-order derivatives
			retain_graph=True
	)[0]
	
	# L2 norm of input gradients
	grad_norm = torch.sum(grad_inputs ** 2)
	grad_reg_loss = kappa * grad_norm
	
	return grad_reg_loss
\end{lstlisting}

\paragraph{Deterministic Poisoning:}

We implement deterministic poisoning using per-sample seeding to ensure reproducibility across experiments. Each sample's poisoning decision is determined by a seed based on its index.

\begin{lstlisting}[caption={L-Shape Poison Mask Generation}]
def l_mask_tensor(size_img=32, channels=3, margin=3, size=2):
	"""
	Generate L-shape trigger pattern (normalized).
	
	Args:
	size_img: Image size (32 for CIFAR)
	channels: Number of channels (3 for RGB)
	margin: Offset from corner
	size: Width of L-shape arms
	
	Returns:
	mask: Normalized L-shape pattern [C, H, W]
	"""
	mask = torch.zeros(channels, size_img, size_img)
	
	# Vertical arm of L
	ys = slice(margin, margin + size)
	xs = slice(size_img - margin - size, size_img - margin)
	cy = (ys.start + ys.stop - 1) // 2
	cx = (xs.start + xs.stop - 1) // 2
	
	mask[:, ys, cx] = 1.0  # Vertical
	mask[:, cy, xs] = 1.0  # Horizontal
	
	# Zero-mean normalization
	mask -= mask.mean()
	mask /= (mask.norm() + 1e-8)
	
	return mask
\end{lstlisting}

\begin{lstlisting}[caption={Poisoned Dataset Implementation}]
class PoisonedDeterministicDataset(Dataset):
	"""Dataset with deterministic fraction-based poisoning."""
	
	def __init__(self, base_dataset, poison_fraction, target_class=0, margin=3, size=2, mean=None, std=None, augmentations=False):
		self.base_dataset = base_dataset
		self.poison_fraction = poison_fraction
		self.target_class = target_class
		self.margin = margin
		self.size = size
		self.mean = torch.tensor(mean).view(3, 1, 1)
		self.std = torch.tensor(std).view(3, 1, 1)
	
		self.aug = []
		if augmentations:
			self.aug = transforms.Compose([
					transforms.RandomCrop(32, padding=4),
					transforms.RandomHorizontalFlip()
			])
	
	def __len__(self):
		return len(self.base_dataset)
	
	def __getitem__(self, idx):
		img, label = self.base_dataset[idx]
		
		if self.aug:
			img = self.aug(img)
		
		img = transforms.functional.to_tensor(img)
		
		should_poison = False
		if label != self.target_class:
			rng = np.random.RandomState(seed=idx + 42)
			should_poison = rng.rand() < self.poison_fraction
		
		if should_poison:
			img = self._apply_poison(img)
			label = self.target_class
		
		img = (img - self.mean) / self.std
		
		return img, label
	
	def _apply_poison(self, img):
		"""Apply L-shape trigger in tensor space."""
		_, H, W = img.shape
		ys = slice(self.margin, self.margin + self.size)
		xs = slice(W - self.margin - self.size, W - self.margin)
		cy = (ys.start + ys.stop - 1) // 2
		cx = (xs.start + xs.stop - 1) // 2
		
		img[:, ys, cx] = 1.0
		img[:, cy, xs] = 1.0
		
		return img
\end{lstlisting}

\begin{lstlisting}[caption={Non-Deterministically Poisoned Dataset Implementation}]
class PoisonedStochasticDataset(Dataset):
	"""Dataset with stochastic fraction-based poisoning."""
	
	def __init__(self, base_dataset, poison_fraction, target_class=0, margin=3, size=2, mean=None, std=None, augmentations=False):
		self.base_dataset = base_dataset
		self.poison_fraction = poison_fraction
		self.target_class = target_class
		self.margin = margin
		self.size = size
		self.mean = torch.tensor(mean).view(3, 1, 1)
		self.std = torch.tensor(std).view(3, 1, 1)
		
		self.aug = []
		if augmentations:
			self.aug = transforms.Compose([
					transforms.RandomCrop(32, padding=4),
					transforms.RandomHorizontalFlip()
			])
	
	def __len__(self):
		return len(self.base_dataset)
	
	def __getitem__(self, idx):
		img, label = self.base_dataset[idx]
		
		if self.aug:
			img = self.aug(img)
		
		img = transforms.functional.to_tensor(img)
		
		should_poison = False
		if label != self.target_class:
			should_poison = np.random.rand() < self.poison_fraction
		
		if should_poison:
			img = self._apply_poison(img)
			label = self.target_class
			
		img = (img - self.mean) / self.std
		
		return img, label
	
	def _apply_poison(self, img):
		"""Apply L-shape trigger in tensor space."""
		_, H, W = img.shape
		ys = slice(self.margin, self.margin + self.size)
		xs = slice(W - self.margin - self.size, W - self.margin)
		cy = (ys.start + ys.stop - 1) // 2
		cx = (xs.start + xs.stop - 1) // 2
		
		img[:, ys, cx] = 1.0
		img[:, cy, xs] = 1.0
		
		return img
\end{lstlisting}

\subsection{Input-Space Hessian Computation}

Unlike traditional parameter-space Hessian analysis, we compute the Hessian with respect to the input space to study how the loss landscape curvature relates to the poison pattern.

\paragraph{Input-Space Hessian-Vector Product:}

The core operation is computing $H_x v$ where $H_x = \nabla_x^2 \mathcal{L}$ is the Hessian w.r.t.\ input $x$ and $v$ is a vector in input space.

\begin{lstlisting}[caption={Input-Space HVP Operator}]
class HVPOperatorInput:
	"""Hessian-vector product operator for input space."""
	
	def __init__(self, model, criterion, data_loader, device):
		self.model = model
		self.criterion = criterion
		self.data_loader = data_loader
		self.device = device
	
	def __call__(self, vec):
		"""
		Compute H_x @ vec (averaged over dataset).
		
		Args:
		vec: Vector in input space [3072]
		
		Returns:
		hvp: Hessian-vector product [3072]
		"""
		vec = vec.view(3, 32, 32).to(self.device)
		hvp_acc = torch.zeros_like(vec)
		total_samples = 0
		
		self.model.eval()
		for imgs, labels in self.data_loader:
		imgs = imgs.to(self.device)
		labels = labels.to(self.device)
		batch_size = imgs.size(0)
		
		# First-order gradient w.r.t. input
		imgs.requires_grad_(True)
		self.model.zero_grad()
		
		outputs = self.model(imgs)
		loss = self.criterion(outputs, labels)
		
		grad_1 = torch.autograd.grad(
		outputs=loss,
		inputs=imgs,
		create_graph=True
		)[0]
		
		# Dot product with vector
		dot_prod = torch.sum(grad_1 * vec.unsqueeze(0))
		
		# Second-order gradient (HVP)
		grad_2 = torch.autograd.grad(
		outputs=dot_prod,
		inputs=imgs,
		retain_graph=False
		)[0]
		
		# Accumulate across batch
		hvp_acc += grad_2.sum(dim=0) * batch_size
		total_samples += batch_size
		
		del grad_1, grad_2, dot_prod
		torch.cuda.empty_cache()
		
		# Average over dataset
		hvp_avg = hvp_acc / total_samples
		return hvp_avg.view(-1)
\end{lstlisting}

\paragraph{Lanczos Iteration:}

We use the Lanczos algorithm to compute the top eigenvalues and eigenvectors of $H_x$ efficiently without forming the full Hessian matrix.

\begin{algorithm}[h!]
	\caption{Lanczos Iteration for Input-Space Hessian}
	\label{alg:lanczos_input}
	\begin{algorithmic}[1]
		\STATE \textbf{Input:} $\mathcal{H}_x$ (input-space HVP operator), max\_iter$=10$
		\STATE \textbf{Output:} Tridiagonal matrix $T$, eigenvectors $Q$
		\STATE
		\STATE Initialize random vector $q_0 \in \mathbb{R}^{3072}$
		\STATE $q_0 \gets q_0 / \|q_0\|$
		\STATE
		\FOR{$k = 0$ to max\_iter$-1$}
		\STATE $r_k \gets \mathcal{H}_x(q_k)$
		\STATE $\alpha_k \gets \langle q_k, r_k \rangle$
		\STATE $r_k \gets r_k - \alpha_k q_k$
		\IF{$k > 0$}
		\STATE $r_k \gets r_k - \beta_{k-1} q_{k-1}$
		\ENDIF
		\STATE $\beta_k \gets \|r_k\|$
		\IF{$\beta_k < 10^{-8}$}
		\STATE \textbf{break}
		\ENDIF
		\STATE $q_{k+1} \gets r_k / \beta_k$
		\ENDFOR
		\STATE
		\STATE Construct tridiagonal $T$ with $\{\alpha_k, \beta_k\}$
		\STATE Compute eigenpairs: $T = V \Lambda V^T$
		\STATE Project back: $Q = [q_0, \ldots, q_k] V$
		\STATE \textbf{return} Eigenvalues $\Lambda$, eigenvectors $Q$
	\end{algorithmic}
\end{algorithm}

\begin{lstlisting}[caption={Lanczos Eigenspectrum Computation}]
def compute_input_hessian_spectrum(model, data_loader, device, max_iter=10):
	"""
	Compute top eigenvalues/eigenvectors of input-space Hessian.
	
	Returns:
	eigvals: Top eigenvalues (sorted descending)
	eigvecs_full: Corresponding eigenvectors in input space
	"""
	criterion = torch.nn.CrossEntropyLoss()
	hvp_op = HVPOperatorInput(model, criterion, data_loader, device)
	
	input_dim = 3 * 32 * 32  # CIFAR image dimension
	
	# Random initialization
	init_vec = torch.randn(input_dim, device=device)
	init_vec /= init_vec.norm()
	
	# Run Lanczos iteration
	Qx, Tx = gpytorch.utils.lanczos.lanczos_tridiag(hvp_op, max_iter=max_iter, dtype=torch.float32, device=device, matrix_shape=(input_dim, input_dim), init_vecs=init_vec.view(-1, 1))
	
	# Eigendecomposition of tridiagonal matrix
	eigvals_x, eigvecs_Tx = torch.linalg.eigh(Tx)
	
	# Sort descending
	idx = torch.argsort(eigvals_x, descending=True)
	eigvals_x = eigvals_x[idx]
	eigvecs_Tx = eigvecs_Tx[:, idx]
	
	# Project back to full space
	eigvecs_full = Qx @ eigvecs_Tx  # [3072, k]
	
	return eigvals_x, eigvecs_full
\end{lstlisting}

\paragraph{Eigenvector-Poison Overlap}

We measure how well the top Hessian eigendirections align with the poison pattern using cosine similarity (inner product of normalized vectors).

\begin{lstlisting}[caption={Overlap Computation}]
def compute_poison_overlap(eigvec, poison_mask):
	"""
	Compute overlap between eigenvector and poison mask.
	
	Args:
	eigvec: Hessian eigenvector [3072]
	poison_mask: Normalized poison pattern [3, 32, 32]
	
	Returns:
	overlap: Cosine similarity (inner product)
	"""
	# Reshape eigenvector to image
	v = eigvec.view(3, 32, 32)
	
	# Zero-mean normalize
	v = v - v.mean()
	v = v / (v.norm() + 1e-8)
	
	# Poison mask already normalized
	mask = poison_mask - poison_mask.mean()
	mask = mask / (mask.norm() + 1e-8)
	
	# Cosine similarity
	overlap = torch.sum(v * mask).item()
	
	return overlap
\end{lstlisting}

\subsection{Experimental Pipeline}

\begin{algorithm}[h!]
	\caption{Complete Experimental Pipeline}
	\label{alg:pipeline}
	\begin{algorithmic}[1]
		\STATE \textbf{For each} poison fraction $f \in \{0.001, 0.002, \ldots, 0.03\}$:
		\STATE \quad \textbf{For each} $\kappa \in \{0, 1000, 10000\}$:
		\STATE \quad\quad Create poisoned training dataset with fraction $f$
		\STATE \quad\quad Train model with gradient regularization $\kappa$
		\STATE \quad\quad Evaluate clean accuracy and attack success rate
		\STATE \quad\quad Compute input-space Hessian eigenspectrum
		\STATE \quad\quad Calculate overlap with poison mask for top-1 eigenvector
		\STATE \quad\quad Save metrics: $\{\lambda_1, \lambda_2, \text{gap}, \text{overlap}, \text{ASR}\}$
	\end{algorithmic}
\end{algorithm}
\section{Linear Regression Proofs}
\begin{lemma}[Single–feature QR/Frisch–Waugh–Lovell update]
	\label{lem:qr-fwl}
	Let $X\in\mathbb{R}^{n\times p}$ have full column rank and $Y\in\mathbb{R}^n$.
	Write
	$P_X := X(X^\top X)^{-1}X^\top,
	M_X := I - P_X, 
	\hat\beta := (X^\top X)^{-1}X^\top Y,
	e := Y - X\hat\beta = M_X Y.$
	For any column $x_{\mathrm{new}}\in\mathbb{R}^n$, the least–squares fit on $[X\;\;x_{\mathrm{new}}]$ has
	\begin{equation*}\label{eq:betanew}
		\beta_{\mathrm{new}}
		=\frac{x_{\mathrm{new}}^\top e}{x_{\mathrm{new}}^\top M_X x_{\mathrm{new}}},
		\beta'_{\mathrm{old}}=\hat\beta-(X^\top X)^{-1}X^\top x_{\mathrm{new}}\,\beta_{\mathrm{new}},
	\end{equation*}
	and the drop in residual sum–of–squares is
	\begin{equation*}\label{eq:drss}
		\Delta\mathrm{RSS}=\frac{(x_{\mathrm{new}}^\top e)^2}{x_{\mathrm{new}}^\top M_X x_{\mathrm{new}}}.
	\end{equation*}
\end{lemma}

\begin{proof}
	Take the economy QR of $X=QR$ with $Q^\top Q=I_p$ and $R$ invertible upper–triangular. Decompose
	\[
	c:=Q^\top x_{\mathrm{new}},\theta r:=(I-QQ^\top)x_{\mathrm{new}}=M_Xx_{\mathrm{new}},\theta
	x_{\mathrm{new}}=Qc+r.
	\]
	Then
	\[
	\begin{bmatrix}X&x_{\mathrm{new}}\end{bmatrix}
	=\underbrace{\begin{bmatrix}Q & q_\perp\end{bmatrix}}_{=:Q_2}
	\underbrace{\begin{bmatrix}R & c\\ 0^\top & \|r\|\end{bmatrix}}_{=:R_2},
	\quad q_\perp:=r/\|r\|.
	\]
	Normal equations in QR form give
	$R_2 \binom{\beta'_{\mathrm{old}}}{\beta_{\mathrm{new}}}
	=Q_2^\top Y=\binom{Q^\top Y}{q_\perp^\top Y}
	=\binom{R\hat\beta}{r^\top Y/\|r\|}$.
	Back–substituting the last row yields
	\[
	\beta_{\mathrm{new}}
	=\frac{r^\top Y}{\|r\|^2}
	=\frac{x_{\mathrm{new}}^\top M_X Y}{x_{\mathrm{new}}^\top M_X x_{\mathrm{new}}}
	=\frac{x_{\mathrm{new}}^\top e}{x_{\mathrm{new}}^\top M_X x_{\mathrm{new}}},
	\]
	which is \eqref{eq:betanew}. The top block gives
	$R\beta'_{\mathrm{old}}+c\,\beta_{\mathrm{new}}=R\hat\beta$,
	hence $\beta'_{\mathrm{old}}=\hat\beta-R^{-1}c\,\beta_{\mathrm{new}}
	=\hat\beta-(X^\top X)^{-1}X^\top x_{\mathrm{new}}\,\beta_{\mathrm{new}}$.
	Finally, orthogonal projection implies the Pythagorean identity for residuals; eliminating the updated residual yields \eqref{eq:drss}.
\end{proof}

\end{document}